\documentclass[10pt,journal]{IEEEtran}
\usepackage[utf8]{inputenc} 
\usepackage[T1]{fontenc}    
\usepackage{url}            
\usepackage{booktabs}       
\usepackage{amsfonts}       
\usepackage{nicefrac}       
\usepackage{microtype}      
\usepackage[table]{xcolor}
\usepackage[utf8]{inputenc}
\usepackage{amsmath}
\usepackage{algorithm}
\usepackage{algorithmic}

\usepackage{fixltx2e}

\usepackage{graphicx}

\usepackage{subcaption}              

\usepackage[utf8]{inputenc}

\usepackage{amsfonts}
\usepackage{amsmath}
\usepackage{mathtools}

\usepackage{amssymb}

\usepackage{bm}

\usepackage{color,verbatim}
\usepackage{multirow}
\usepackage{accents}

\usepackage{theoremref}

\usepackage{flushend}

\usepackage{url}

\newcounter{rulecounter}
\newcommand{\resetrule}{ \setcounter{rulecounter}{0}}
\resetrule

\newsavebox{\selvestebox}
\newenvironment{colbox}[1]
  {\newcommand\colboxcolor{#1}%
   \begin{lrbox}{\selvestebox}%
   \begin{minipage}{\dimexpr\columnwidth-2\fboxsep\relax}}
  {\end{minipage}\end{lrbox}%
   \begin{center}
   \colorbox{\colboxcolor}{\usebox{\selvestebox}}
   \end{center}}

\definecolor{orange}{rgb}{1,0.8,0}
\definecolor{gray}{rgb}{.9,0.9,0.9}
\definecolor{darkgray}{rgb}{.3,0.3,0.3}
\definecolor{darkblue}{rgb}{.1,0.0,0.3}
\definecolor{lightblue}{rgb}{0.7,0.7,1}
\definecolor{lightred}{rgb}{1,0.7,.7}
\definecolor{purple}{RGB}{204,153,255}
\definecolor{lightgray}{rgb}{.95,0.95,0.95}
\definecolor{lightgreen}{rgb}{0.3,0.5,0.3}
\definecolor{darkgreen}{rgb}{0.05,0.3,0.05}

\newtheorem{myproposition}{Proposition}
\newtheorem{myremark}{Remark}
\newtheorem{myproblemstatement}{Problem Statement}
\newtheorem{mylemma}{Lemma}
\newtheorem{mytheorem}{Theorem}

\newtheorem{mycorollary}{Corollary}

\usepackage{pgfplots}
\pgfplotsset{compat=newest}
\usetikzlibrary{plotmarks}
\usetikzlibrary{arrows.meta}
\usepgfplotslibrary{patchplots}
\usepackage{grffile}
\usetikzlibrary{pgfplots.groupplots}
\pgfplotsset{plot coordinates/math parser=false}
\usetikzlibrary{pgfplots.units}
\pgfplotsset{plot coordinates/math parser=false}
\newlength\mywidth
\newlength\myheight
\definecolor{GraphSAC}{RGB}{228,26,28}%
\definecolor{GraphSAChk}{RGB}{0,26,208}%
\definecolor{Amen}{RGB}{55,126,184}%
\definecolor{Gae}{RGB}{77,175,74}
\definecolor{Radar}{RGB}{152,78,163}
\definecolor{Cut}{RGB}{255,127,0}%
\definecolor{Flake}{RGB}{153,153,153}%
\definecolor{Conductance}{RGB}{166,86,40}
\definecolor{Avg}{RGB}{247,129,191}

\definecolor{CGTF}{rgb}{1,0.04700,0.04100}%
\definecolor{SNMF}{rgb}{0.05000,0.32500,0.89800}%
\definecolor{CGTF1}{rgb}{0.92900,0.69400,0.12500}%
\definecolor{SNMF1}{rgb}{0.49400,0.18400,0.55600}%
\definecolor{CMTF}{rgb}{0.46600,0.67400,0.18800}%
\definecolor{PARAFAC}{rgb}{0.30100,0.74500,0.93300}%
\definecolor{NTF}{rgb}{1,1,0.08400}%
\newcommand{\marginl}{0.1}
\newcommand{\margin}{0.4}
\newcommand{\mylinewidth}{1.5}
\newcommand{\markwidth}{1.5}
\newcommand{\legendfontsize}{\small}

\usepackage{amsthm}
\newtheorem{theorem}{Theorem}
\newtheorem{corollary}{Corollary}

\newcommand{\transpose}{\top}
\usepackage{xr}

\usepackage{wrapfig}
\usepackage{tikz}
\usetikzlibrary{shapes,arrows}
\usetikzlibrary{fit,chains}
\title{GraphSAC: Detecting anomalies in large-scale graphs}

\author{Vassilis N. Ioannidis$^\star$, Dimitris Berberidis$^\dagger$  and
Georgios B. Giannakis$^\star$
\thanks{
The work in this paper has been supported by the Doctoral Dissertation Fellowship of the Univ. of Minnesota, the USA NSF grants 171141,  1500713, and  1442686.}
\\$^\star$ Digital Technology Center and Dept. of ECE, University of Minnesota, Minneapolis, USA\\
$^\dagger$ Heinz College of Information Systems and Public Policy, Carnegie Mellon University, Pittsburgh, USA\\
E-mails: $\{$ioann006, georgios$\}$ @umn.edu
}
\begin{document}

\maketitle
\begin{abstract}
A graph-based sampling and consensus (GraphSAC) approach is introduced to effectively  detect  anomalous nodes in large-scale graphs. Existing approaches rely on connectivity and attributes of all nodes to assign an anomaly score per node. However, nodal attributes and network links might be compromised by adversaries, rendering these holistic approaches vulnerable. Alleviating this limitation, GraphSAC randomly draws subsets of nodes, and relies on graph-aware criteria to judiciously filter out sets contaminated by anomalous nodes, before employing a semi-supervised learning (SSL) module to estimate nominal label distributions per node.
These learned nominal distributions are minimally affected by the anomalous nodes, and hence can be directly adopted for anomaly detection. Rigorous analysis provides performance guarantees for GraphSAC, by bounding the required number of draws. The per-draw complexity grows linearly with the number of edges, which implies efficient SSL, while draws can be run in parallel, thereby ensuring scalability to large graphs. GraphSAC is tested under different anomaly generation models based on random walks, clustered anomalies, as well as contemporary adversarial attacks for graph data. Experiments with real-world graphs showcase the advantage of GraphSAC relative to state-of-the-art alternatives.
\end{abstract}

\section{Introduction}
The ever-expanding interconnection of social, email, and media service platforms presents an opportunity for adversaries manipulating networked data to launch malicious attacks~\cite{goodfellow2014explaining,aggarwal2015outlier,yan2018deep,pang2018towards}. Adversarially perturbed or simply anomalous graph data may disrupt the operation of critical machine learning {algorithms} with severe consequences. Detecting anomalies in graph data is of major importance in a number of contemporary applications such as flagging ``fake news,'' unveiling malicious users in social networks, blocking spamming users in email networks, and uncovering suspicious transactions in financial or e-commerce networks~\cite{akoglu2015graph,yan2016robust}.  Detecting these anomalous nodes can be formulated as a learning task over an attributed graph.

Before positioning our work in context, we highlight different types of graph-based anomalies. \emph{Homophilic} anomalies characterize nodes whose attributes are dissimilar to those of their  neighbors~\cite{perozzi2016scalable,mcpherson2001birds}. These nodes violate the homophily property that postulates neighboring vertices to have similar attributes, and is heavily employed in semi-supervised learning (SSL)~\cite{smola2003kernels,kipf2016semi,gleich2015pagerank,kloster2014heat}. In a social network of voters for example, friends typically belong to the same voting party; see Fig.~\ref{fig:anom_vot}. Oftentimes, anomalous nodes may form dense connections giving rise  to clustered homophilic anomalies Fig.~\ref{fig:anom_clust}.
\emph{Structural} anomalies correspond to nodes with attributes that are dissimilar to structurally similar nodes~\cite{eberle2007discovering}. 
Structural similarity among nodes suggests that vertices involved in similar graph structural patterns possess related attributes~\cite{donnat2018learning}. In an academic collaboration network for instance, nodes with similar graph structure (central nodes) have similar labels (e.g., professors); see Fig.~\ref{fig:anom_col}. 

\begin{figure*}[t]
    \centering
    \begin{subfigure}[b]{0.25\textwidth}\hspace{-1cm}
        \includegraphics[width=1\textwidth]{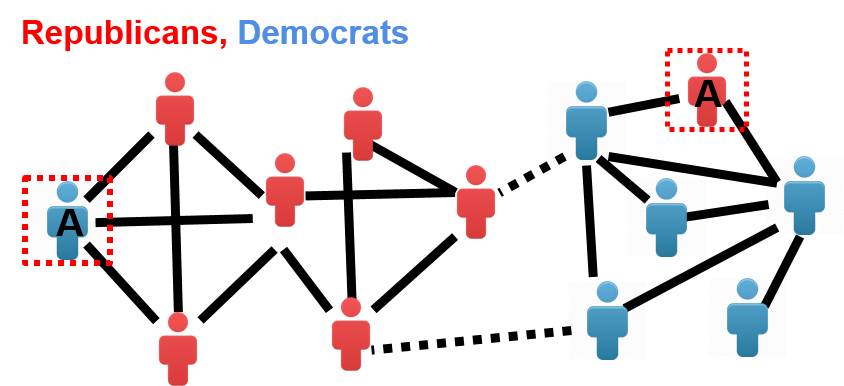}
        \caption{Social network of voters.}
        \label{fig:anom_vot}
    \end{subfigure}~
     \begin{subfigure}[b]{0.25\textwidth}
        \includegraphics[width=0.95\textwidth]{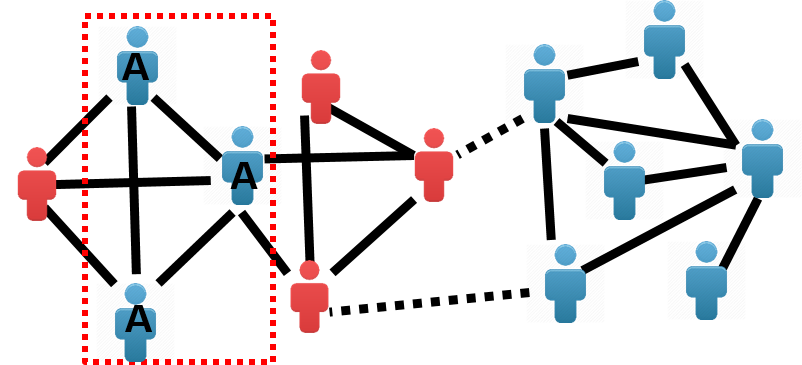}
        \caption{Clustered anomalies.}
        \label{fig:anom_clust}
    \end{subfigure}~
    \begin{subfigure}[b]{0.23\textwidth}\hspace{+0.5cm}
        \includegraphics[width=1\textwidth]{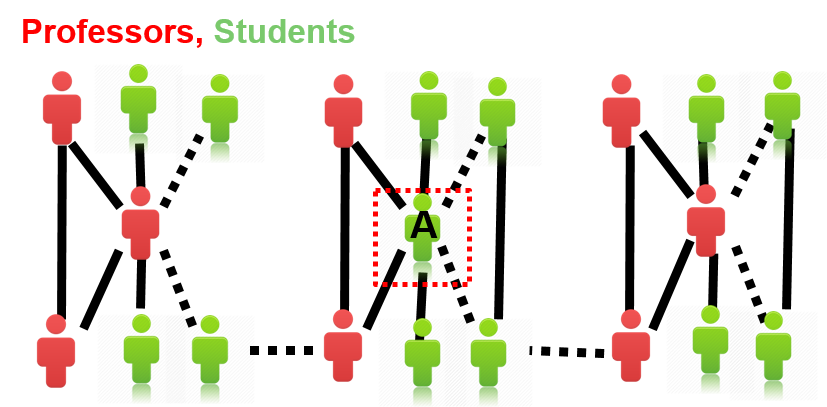}
        \caption{Academic network.}
        \label{fig:anom_col}
    \end{subfigure}
    \caption{Nodes in dotted square exhibit (a) (b) homophilic and (c) structural anomalies.}
    \label{fig:anom}
\end{figure*}

\subsection{Related Work}
Todays era of data deluge has grown the interest for detecting anomalies in collections of high-dimensional data~\cite{iwata2016multi,zimek2012survey}.  This paper deals with anomalies in data that exhibit inter-dependencies captured by a graph~\cite{akoglu2015graph}. The inaccessibility and prohibitive cost associated with obtaining ground-truth anomalies motivates the development of mainly unsupervised techniques.

Methods for detecting anomalies in attributed graphs can be roughly classified in three categories. \emph{Community-based} approaches find clusters of nodes and search for anomalies within each cluster. A probabilistic method is developed in \cite{gao2010community} that jointly discovers communities, and detects community outliers as anomalies. Similarly, \cite{perozzi2016scalable} identifies anomalies by measuring the attribute correlation of nodes within each node's egonet, meaning the subgraph induced by the node of interest, its one-hop neighbors, and all their connections. \emph{Subspace-based} approaches focus on spotting anomalies in subspaces extracted from the nodal features~\cite{sanchez2013statistical}. 
On the other hand, \emph{model-based} methods learn an embedding per node and flag anomalies by measuring the model-fitting  error~\cite{li2017radar,ding2019deep}. A parametric model is developed in \cite{li2017radar} to capture the coherence among the attributes of nodes and their connectivity. A deep graph autoencoder is advocated in \cite{ding2019deep} that fuses attributes and connections to an embedding per node, and identifies anomalies using the reconstruction error at the decoder side. Despite their empirical success, these contemporary approaches are confronted with a number of challenges. The computational overhead associated with community detection, subspace extraction and deep learning, discourages their applicability to large-scale graphs. The 
local scope of community-based methods confines the breadth of the anomaly detector that is further vulnerable to clusters of connected anomalous nodes. Finally, all aforementioned approaches ultimately learn an anomaly score that relies on attributes and connections of \emph{all} nodes. However, either the attributes or the network links for some nodes may be compromised by adversaries~\cite{zugner18adv,xu2019topology}.

\subsection{Contributions}
Addressing the aforementioned challenges, we introduce a graph random sampling and consensus (GraphSAC) framework for detecting anomalous nodes on large graphs. Instead of directly considering all nodes, our novel method samples subsets of nodes, and relies on graph-aware criteria to judiciously filter out subsets contaminated by anomalous nodes. The ``clean'' sets are utilized by a SSL module that estimates a nominal class distribution per node. The core intuition behind GraphSAC is that attributes of anomalous nodes will have poor predictive performance in a SSL task. The contribution of this work is fourfold.
i) A novel approach to estimating a class distribution per node that is guaranteed to be minimally affected by anomalous nodes; ii) A versatile framework that adapts to different types of anomalies via an application-specific SSL module (cf. Fig.~\ref{fig:anom});  
iii) Scalability to large-scale graphs (complexity is linear in the number of edges); and iv) experimental evidence confirming that the novel graph anomaly detector outperforms state-of-the-art approaches in identifying clusters of anomalous nodes, as well as contemporary adversarial attacks on graph data.
	
\section{Graph-based Random Sampling and Consensus} Consider a graph $\mathcal{G}:=\{\mathcal{V},\mathcal{E}\}$, where $\mathcal{V}:=\{n_1,n_2,\ldots,n_N\}$ is the vertex set, and $\mathcal{E}$ the edge set of $E$ edges . The connectivity of $\mathcal{G}$ is described by an adjacency matrix ${\mathbf A}\in{\mathbb R}^{ N   \times N   }$, where $[{\mathbf A}]_{n,n'}>0$  if $(n,n')\in \mathcal{E}$. Each node $n\in\mathcal{V}$ is associated with one or more scalar labels $y_n\in\{1,\ldots,C\}$ that form the $N\times C$ matrix $\mathbf{Y} :=[\mathbf{y}_1^\transpose,\ldots,\mathbf{y}_N^\transpose]^\transpose$ with ${[\mathbf{Y}]}_{n,c}=1$, if $y_n=c$, and $0$ otherwise.

Given ${\mathbf A}$ and $\mathbf{Y}$ the goal in this paper is to detect ${K}$ anomalous nodes with indices in the set $\mathcal{A}:=\{{n}_{1},\ldots,{n}_{{K}}\}$. Such nodes are expected to violate a certain property such as homophily. To this end, we require a model that relates the graph with the labels, and promotes the desired properties. 

An immediate approach is to directly consider all nodal labels and connections in a graph-based model. However, such a holistic approach will be vulnerable to the inclusion of anomalous nodes that will bias the learned model and poison the anomaly detection framework. 

Instead, our idea is to sample labels $y_n$ at random subsets of nodes $n\in\mathcal{L}\subset \mathcal{V}$ and prudently discard contaminated subsets. Given $\mathcal{L}$,
we  perform SSL to predict the labels across all nodes. SSL methods utilize the labels at ${\mathcal L}$ along with the graph connectivity ${\mathbf A}$ to predict the labels at the unlabeled nodes $\mathcal{V}\setminus\mathcal{L}$. We draw inspiration from the random sampling approach for robust model fitting in image analysis~\cite{bolles1981ransac}. The SSL model $f(\cdot)$ utilizes the labels in $\mathcal L$ to estimate the $N\times C$ label distribution matrix as follows
\begin{align} 
\label{eq:ssl}
{\hat{\mathbf P}}=f(\{y_n\}_{n\in\mathcal L},\mathbf{A})
\end{align}
where $\hat{{P}}_{(n,c)}\in[0,1]$ can be interpreted as the probability that $y_n=c$. Henceforth, for notation brevity we define the SSL model as follows $f({\mathcal L}):=f(\{y_n\}_{n\in\mathcal L},\mathbf{A})$. The choice of $f(\cdot)$ is dictated by specific properties one may want to capture; see also Fig. \ref{fig:anom}. 

Nevertheless, if $\mathcal{L}\cap\mathcal{A}\ne\emptyset$, the predicted label distributions will be affected by the anomalous nodes. To bypass this hurdle, we formulate a hypotheses test to assess if anomalies are present in $\mathcal{L}$ by evaluating the predictive SSL performance instantiated with $\mathcal{L}$, namely $f({\mathcal L})$. Our test relies on the premise that attributes of anomalous nodes will have poor predictive performance for SSL. 

\begin{algorithm}[t]
\caption{GraphSAC}
\begin{algorithmic}
\STATE\textbf{Input:} $f(\cdot,\cdot)$, $\mathbf{A}$, $\{y_{{n}}\}_{{n}=1}^ N, I,  T, i\leftarrow0$
\STATE 1. \textbf{while}  $i<I$ \textbf{do}
\STATE 2. \hspace{0.1cm} Select $\mathcal{L}^{(i)}$ at random
\STATE 3. \hspace{0.1cm} Estimate ${\hat{\mathbf P}}^{(i)}_{G}=f(\mathcal{L}^{(i)})$ and  consensus set  $\mathcal{U}^*$
\STATE 4. \hspace{0.1cm} If $|\mathcal{U}^*|/N<T$ then $\delta_f(\mathcal{L}^{(i)})=0$,  otw. $\delta_f(\mathcal{L}^{(i)})=1$
\STATE 5. \hspace{0.1cm} $i\leftarrow i+1$
\STATE 6. \textbf{end while}
\STATE 7. Obtain $\hat{\mathbf{P}}_{G}$ as in \eqref{eq:grout} 
\STATE 8. Obtain anomaly scores $\{\phi_{n}\}_{n=1}^N$ as in \eqref{eq:anoma_score} 
\end{algorithmic}
\end{algorithm}

Our iterative algorithm termed GraphSAC is summarized as Algorithm 1. Per iteration $i$, we first sample $S$ nodes uniformly at random from $\mathcal{V}$ without replacement, that is 
\begin{align}
\label{eq:samplgrsac}
\mathcal{L}^{(i)}\sim \mathrm{Unif}(\mathcal{L}_S)
\end{align}
where $\mathcal{L}_S:=\{ \mathcal{L}\subseteq \mathcal{V} : |\mathcal{L}| = S \}$ is the set of all $S-$size subsets.
Given the labels in $\mathcal{L}^{(i)}$,  the SSL model~\eqref{eq:ssl} outputs the predicted label distribution matrix ${\hat{\mathbf P}}^{(i)}_{G}:=f({\mathcal L}^{(i)})$.
Nodes whose labels are correctly predicted by $f({\mathcal L}^{(i)})$ form the consensus set  $\mathcal{U}^*:=\{n: c'=\arg\max_{c}{[\hat{\mathbf P}}^{(i)}_{G}]_{n,c},~\text{ and }~y_n=c'\}$. 

Next, GraphSAC compares the accuracy of $f({\mathcal L}^{(i)})$ using e.g., the ratio of nodes in the consensus set to a {prespecified} threshold $T$. If $|\mathcal{U}^*|/N>T$, GraphSAC decides that ${\mathcal L}^{(i)}$ does not contain anomalies, meaning $\delta_f(\mathcal{L}^{(i)})=1$; otherwise, the set is contaminated with anomalies and filtered out, that is $\delta_f{(\mathcal{L}^{(i)})}=0$. The following test of hypotheses is performed
\begin{align}
\label{eq:graphsacdetect}
    \left\{
\begin{array}{ll}
    H_0:~~\delta_f(\mathcal{L}^{(i)})=1,~~~~\text{if}~~~|\mathcal{U}^*|/N>T \\
    H_1:~~\delta_f(\mathcal{L}^{(i)})=0,~~~~\text{otherwise}.
\end{array} 
\right. 
\end{align}
Essentially, this test filters out subsets that are contaminated with anomalies $\mathcal{L}\cap\mathcal{A}\ne\emptyset$, and will bias the learned model.
Hence, $\delta_f(\cdot)$ corresponds to a filter that aims to retain only ``clean'' sets i.e. $\mathcal{L}\cap\mathcal{A}=\emptyset$.
We will elaborate on the performance of this filter in Section 3. The resulting sample average of the nominal label distribution is given by  
\begin{align}
\label{eq:grout}
	    \hat{\mathbf{P}}_{G}=\frac{\sum_{i=1}^I f({\mathcal L}^{(i)})\delta_f{(\mathcal{L}^{(i)})}}{\sum_{i=1}^I\delta_f{(\mathcal{L}^{(i)})}}.
\end{align}
Even though  $\mathcal{L}^{(i)}$ are drawn uniformly at random \eqref{eq:samplgrsac}, $\delta_f(\cdot)$ introduces a sampling bias towards ``clean'' subsets. Consequently, $\hat{\mathbf{P}}_{G}$ is minimally affected by anomalous nodes and represents the  nominal class distribution. 

Finally, we select as anomalous the nodes with the largest distance between their nominal distribution and their actual labels.   GraphSAC estimates an  $N\times 1$ anomaly score vector $\bm{\phi}(\hat{\mathbf{P}}_{G})$ with entries
\begin{align}
\label{eq:anoma_score}
\phi_{n}=\text{dist}(\hat{\mathbf{p}}_{n}^{G},\mathbf{y}_{n}),~~~~\forall n \in \mathcal{V}
\end{align}
where 
$\hat{\mathbf{p}}_{n}^{G}$ is the $n$-th row of $\hat{\mathbf{P}}_{G}$, ${\hat{ \mathbf{p}}}_{n}^{(i)}$ is the $n$-th row of ${\hat{\mathbf P}}^{(i)}_{G}$,  and dist$(\cdot,\cdot)$ is the cross-entropy loss Therefore, $\phi_{n}$ is larger if $n$ does not adhere to the graph-related properties promoted by the  SSL model. 
Hence, we rank the nodes in decreasing order with respect to $\phi_n$, and select the first K nodes as anomalous. Fig~\ref{fig:encoder} illustrates the GraphSAC operations.

\begin{figure*}\centering
    \input{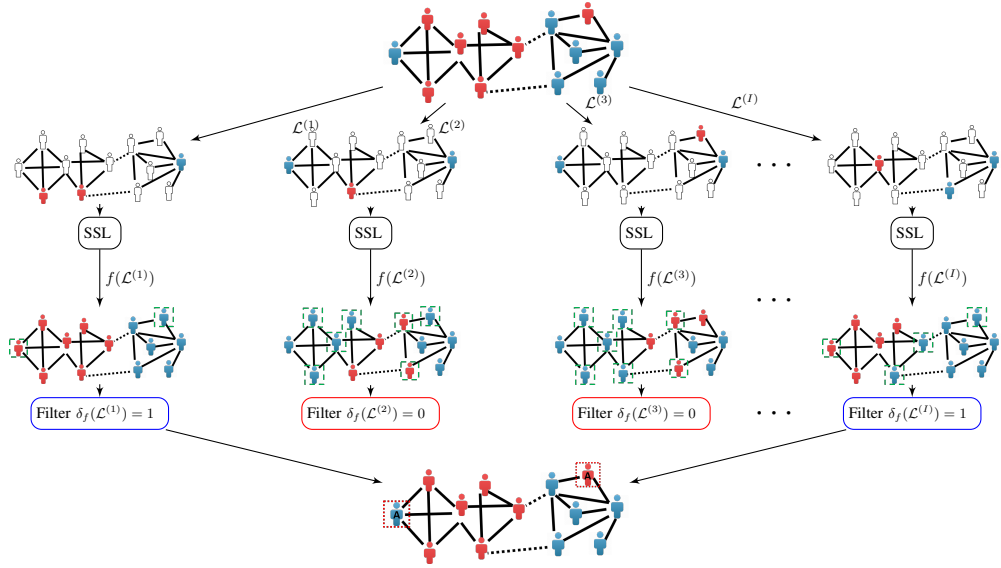}
	\caption{An illustration of the operation of GraphSAC.  The first row represents the available graph and labels. The second row shows the sampled labels. The third row contains the SSL module that outputs the predicted labels, where green dotted lines indicate incorrectly classified nodes. The filter in the fifth row decides whether  $\mathcal{L}^{(i)}$ contains anomalies or not. The predictions with a larger number of missclasified nodes are discarded (colored red). The final anomalies are detected by combining the predictions from the ``clean'' subsets.} \label{fig:encoder}
\end{figure*}

Instead of using as many nodes as possible to obtain the solution, GraphSAC relies on small sets of $S$ nodes and SSL-aided hypotheses testing to avoid subsets contaminated with anomalous nodes. The small sample size ($S\ll N$) enables GraphSAC to remain operational even under adverse conditions where $K$ is relative large. GraphSAC's robustness is justified since only one ``clean'' ${\mathcal L}^{(i)}$ is required for a valid $\hat{\mathbf{P}}_G$~\eqref{eq:grout}. 

The computational complexity of GraphSAC per $i$ is dictated by the label prediction step $f({\mathcal L})$ that scales linearly with the number of edges $\mathcal{O}(E)$ for scalable  SSL  methods~\cite{zhu2003semi,chapelle2006}. Further, since the draws $\mathcal{L}^{(i)}$ are independent, each GraphSAC iteration $i$ can be readily parallelized, thereby ensuring scalability to large-scale graphs. 

Notice that so far $f(\cdot)$ is not specified. Hence, GraphSAC may adapt to the pertinent type of anomalies (see Fig.~\ref{fig:anom}) by appropriately choosing the model $f(\cdot)$. {Homophilic} anomalies for example, call for SSL methods e.g. diffusion-based classifiers \cite{zhu2003semi,chapelle2006} or contemporary graph convolutional neural networks (GCN)s~\cite{kipf2016semi}. On the other hand, {structural} anomalies necessitate models that promote structural similarities among nodes such as the work in~\cite{donnat2018learning}.

\section{Analytical guarantees}
This section strengthens our proposed randomized anomaly detection framework with analytical guarantees. Towards streamlining the analysis, we interpret the filter $\delta_f(\cdot)$ as introducing a sampling bias to the uniform sampling scheme in~\eqref{eq:samplgrsac} towards ``clean'' subsets.\footnote{An alternative analysis accounting directly for $\delta_f(\cdot)$ is in Section D of the supplementary material.}  Specifically, is rewritten $\hat{\mathbf{P}}_G$ as
\begin{align}
\label{eq:altgrout}
    \hat{\mathbf{P}}_{G}&:=\frac{1}{I}\sum_{i=1}^I
    f({\mathcal L}^{(i)}) \\
    \shortintertext{where 
    $\mathcal{L}^{(i)}\sim p_G(\mathcal{L})$ 
     are drawn from the GraphSAC biased sampling scheme
    instead of~\eqref{eq:samplgrsac} with}
\label{eq:grsamplalt}
 p_G(\mathcal{L})&=
    \begin{cases} p_d(\mathcal{L}),& \forall\mathcal{L}\in \bar{\mathcal{L}}_S\\
    p_{f}(\mathcal{L})
    ,&\forall\mathcal{L}\in \bar{\mathcal{L}}_S^c
    \end{cases}    
\end{align}
where  $\bar{\mathcal{L}}_S$ is a set of nodal subsets with no anomalous nodes $\bar{\mathcal{L}}_S:=\{ \mathcal{L}\subseteq {\mathcal{L}}_S , \mathcal{L}\cap \mathcal{A} = \emptyset \}$, while the complementary set $\bar{\mathcal{L}}_S^c$ contains all the remaining size-$S$ subsets  $\bar{\mathcal{L}}_S^c:={\mathcal{L}}_S \setminus  \bar{\mathcal{L}}_S$. Hence, $p_G(\mathcal{L})$  is related to the filter \eqref{eq:graphsacdetect} through 
\begin{align}
\label{eq:con_det_pmf}
p_d(\mathcal{L})&=p(\delta_f(\mathcal{L})=1|\mathcal{L}\in \bar{\mathcal{L}}_S)p(\mathcal{L}\in \bar{\mathcal{L}}_S)\\
p_f(\mathcal{L})&=p(\delta_f(\mathcal{L})=1|\mathcal{L}\in \bar{\mathcal{L}}^c_S)p(\mathcal{L}\in \bar{\mathcal{L}}^c_S)\nonumber
\end{align}
where $p(\mathcal{L}\in \bar{\mathcal{L}}_S)={|\bar{{\mathcal{L}}}_S|}/{|{\mathcal{L}}_S|}$ and $p(\mathcal{L}\in \bar{\mathcal{L}}^c_S)={|\bar{{\mathcal{L}}}_S^c|}/{|{\mathcal{L}}_S|}$ in accordance with the uniform sampling \eqref{eq:samplgrsac}. 
Notice that for this section all $I$ samples are included in the sample average \eqref{eq:altgrout}, but if the sample is contaminated, i.e. $\mathcal{L}\in \bar{\mathcal{L}}^c_S$, $\mathcal{L}$ will have a smaller probability to be sampled.
GraphSAC aims at a class distribution per node that is not affected by the anomalous nodes, but takes into account only the nominal nodes, and can thus be readily utilized for anomaly detection. Hence, the desired probability matrix is 
\begin{align}
\label{eq:altpnraphsac}
    \mathbf{P}_{N}&:=\mathbb{E}_{\mathcal{L}\sim p_N}[f({\mathcal{L}})] \\
    \shortintertext{where ${\mathcal{L}}$ are drawn uniformly 
    from 
    $\bar{\mathcal{L}}_S$ that is}
   \label{eq:altsampl} p_N(\mathcal{L})&=
    \begin{cases} \frac{1}{|\bar{\mathcal{L}}_S|},&\forall\mathcal{L}\in \bar{\mathcal{L}}_S\\
    ~~0,&\forall\mathcal{L}\in \bar{\mathcal{L}}_S^c \;.
\end{cases}
\end{align}
Indeed, $ \mathbf{P}_{N}$ captures the nominal class distribution per node that is not affected by the anomalous nodes and conforms to the properties promoted by the SSL model.
As a result, the largest entries in the anomaly score vector $\bm{\phi}(\mathbf{P}_{N})$ (cf. \eqref{eq:anoma_score}) represent the anomalous nodes.

However, direct estimation of  $\mathbf{P}_{N}$ is not feasible since  $\mathcal{A}$ is unknown. If on the other hand all nodes are directly accounted for, anomalous ones will be included that will in turn bias the learned probability matrix. To obviate this hurdle, 
GraphSAC introduces  \eqref{eq:graphsacdetect} to filter out contaminated subsets, giving rise to the biased sampler in~\eqref{eq:grsamplalt}.
GraphSAC's ultimate target is the expected probability matrix
\begin{align}
    \label{eq:altpgraphsac}
        \mathbf{P}_{G}:= \mathbb{E}_{\mathcal{L}\sim p_G}[ f(\mathcal{L})].
\end{align}
Upon considering $p_f(\mathcal{L})=0,~\forall\mathcal{L}\in \bar{\mathcal{L}}^c_S$ and  $p_d(\mathcal{L})=1/|\bar{\mathcal{L}}_S|,~\forall\mathcal{L}\in \bar{\mathcal{L}}_S$, as well as expanding the expectation in \eqref{eq:altpgraphsac} and \eqref{eq:altpnraphsac}, GraphSAC's expected probability matrix reduces to the desired one, meaning $\mathbf{P}_{G}=\mathbf{P}_{N}$. This corresponds to a perfect filter in \eqref{eq:graphsacdetect} that disregards all contaminated subsets $\delta_f(\mathcal{L})=0, \forall\mathcal{L}\in \bar{\mathcal{L}}^c_S$, and retains all the clean ones $\delta_f(\mathcal{L})=1, \forall\mathcal{L}\in \bar{\mathcal{L}}_S$.\\
\noindent
Thus,  GraphSAC's performance is directly related to the distance among $\hat{\mathbf{P}}_G$ and $\mathbf{P}_{N}$
\begin{align}
    \text{TV}(\hat{\mathbf{P}}_G,\mathbf{P}_{N})&:=\|\hat{\mathbf{P}}_G- \mathbf{P}_{N}\|_1\\&\le
    \label{eq:triang_ineq}\|\mathbf{P}_G- \mathbf{P}_{N}\|_1+
    \|\hat{\mathbf{P}}_G- \mathbf{P}_G\|_1
\end{align}
where \eqref{eq:triang_ineq} follows from the triangle inequality, and $\|\mathbf{Z}\|_1$ represents the sum of the absolute values of the $\mathbf{Z}$ entries.
Adhering to \eqref{eq:triang_ineq}, the following research questions have to be addressed to  characterize the performance of GraphSAC.
\begin{itemize}
\item[\textbf{RQ1}.] How $\|\mathbf{P}_G- \mathbf{P}_{N}\|_1$ relates to the filter performance? and
\item[\textbf{RQ2}.] How  $\|\hat{\mathbf{P}}_G- \mathbf{P}_G\|_1$  evolves as the number of draws $I$ increases?
\end{itemize}
The following analysis aspires to provide tangible answers to aforementioned questions. First, consider a simplified version of $p_G(\mathcal{L})$ given by
\begin{align}
\label{eq:pgrsampfin}
 p_G(\mathcal{L})&=
    \begin{cases} d,& \forall\mathcal{L}\in \bar{\mathcal{L}}_S\\
   {f}
    ,&\forall\mathcal{L}\in \bar{\mathcal{L}}_S^c \;.
    \end{cases}
\end{align}
which implies that any clean (contaminated) subset in $\bar{\mathcal{L}}_S$ ($\bar{\mathcal{L}}^c_S$) has the same sampling probability ${d}$ ($f$). This approximation suggests that GraphSAC samples all contaminated samples with a certain probability, and all non-contaminated with a different one (cf. \eqref{eq:con_det_pmf}). Extended analysis on a refined sampling scheme different from \eqref{eq:pgrsampfin} is included in Section D of the supplementary material, e.g. if a subset contains a large number of anomalies, the probability this sample is rejected is higher relative to a sample with only one anomaly. 

\begin{theorem} (Proof in Section A of the supplementary material.) Let $\mathbf{P}_A:=\mathbb{E}_{\mathcal{L}\sim\text{Unif}(\bar{\mathcal{L}}_S^c)}\big[ f(\mathcal{L})\big]$ denote the expected distribution when anomalies are sampled. 
It then holds for the total variation distance between $\mathbf{P}_{G}$ and $\mathbf{P}_{N}$ that 
\begin{align}
    \|\mathbf{P}_G- \mathbf{P}_{N}\|_1
    =&\frac{|\bar{\mathcal{L}}_S^c|^2}{|{\mathcal{L}}_S|}p_{fa}\|\mathbf{P}_A -\mathbf{P}_{N}| \!\lceil_1\\=&|\bar{\mathcal{L}}_S^c|f\|\mathbf{P}_A -\mathbf{P}_{N}| \!\lceil_1
\end{align}
where $p_{fa}:=p(\delta_f(\mathcal{L})=1|\mathcal{L}\in \bar{\mathcal{L}}^c_S)$ is the probability of false alarms for \eqref{eq:graphsacdetect} that can be also expressed as $p_{fa}=({|{\mathcal{L}}_S|}/{|\bar{{\mathcal{L}}}_S^c|})f  $
with $|\bar{\mathcal{L}}^c_S|={N\choose S} -{N-K\choose S}$ and $|{\mathcal{L}}_S|={N\choose S}$.
\end{theorem}
\noindent Theorem 1 asserts that the desired distance is equal to the distance between the nominal distribution and the one affected by the anomalies scaled by the probability that GraphSAC fails to identify the anomalies. 
An immediate observation is that for a perfect GraphSAC filter it holds that $p_{fa}=0$, and one would obtain the desired $\mathbf{P}_G=\mathbf{P}_{N}$. Another consequence of Theorem 1, is that as the number of anomalies $K$ increases, so does $|\bar{\mathcal{L}}^c_S|$ and  $\|\mathbf{P}_{G}-\mathbf{P}_{N}\|_1$, conditioned that the filter performance does not change. As demonstrated in the experiments, even if anomalous subsets are miss-classified as nominal by the filter for a specific draw, these will contain a small number of anomalies and will not affect the overall anomaly detection performance.

Towards addressing Q2 we utilize concentration inequalities theory applied to random matrices~\cite{tropp2015introduction} and establish the following. 
\begin{theorem}(Proof in Section B of the supplementary material.) 
With $\hat{\mathbf{P}}_{G}=\frac{1}{I}\sum_{i=1}^I\mathbf{P}^{(i)}_{G}$, where  $\mathbf{P}^{(i)}_{G}=f(\mathcal{L}^{(i)})$ and $\mathbf{P}_{G}= \mathbb{E}_{\mathcal{L}\sim p_G}[ f(\mathcal{L})]$, it holds that  
\begin{align}
    &\mathbb{E}\big\{\|\hat{\mathbf{P}}_{G}-{\mathbf{P}}_{G}\|\big\}\le\sqrt{\frac{2N\log(N+C)}{I}}+\frac{2\sqrt{N}\log(N+C)}{3I}\label{eq:normbound}\\
    \shortintertext{and for all $t>0$}
    &\mathbb{P}\big\{\|\hat{\mathbf{P}}_{G}-{\mathbf{P}}_{G}\|\ge t\big\}\le(N+C)\label{eq:propbound}
    \exp{\left(\frac{-I t^2}{N + 2\sqrt{N}t/3}\right)}.
\end{align}
\end{theorem}    
All the parameters in the right hand side of \eqref{eq:normbound} and \eqref{eq:propbound} are either known or controllable, and relate to the network size $N$ and the number of classes $C$. The expectation bound \eqref{eq:normbound} suggests that by increasing the number of draws $I$ the sample average of the random matrices will approach the desired ensemble mean. On the other hand, \eqref{eq:propbound} bounds the tail of the probability distribution of the difference. It can be shown that for $t\le \sqrt{N}$ the tail probability decays as fast as the tail of a Gaussian random variable with variance proportional to $N$, while for $t\ge \sqrt{N}$ it decays as that of an exponential random variable whose mean is proportional to $\sqrt{N}$.

Retracing back to \eqref{eq:triang_ineq}, we deduce that $\|\hat{\mathbf{P}}_G- \mathbf{P}_G\|_1$ decreases for increasing $I$, whereas $\|{\mathbf{P}}_N- \mathbf{P}_G\|_1$ directly depends on the filter performance.  In practice, GraphSAC amounts to ranking the per-node anomaly scores. This ranking may be evident even for relative large distances $\|\hat{\mathbf{P}}_G- \mathbf{P}_{N}\|_1$. 

\subsection{Analysis on diffusion-based SSL models}
The results presented so far hold for any SSL model even for contemporary GCNs. Next, we focus on the class of diffusion-based SSL models $f_{\text{dif}}$, which have documented success in sizable graphs due to their scalability and robustness~\cite{zhu2003semi,chapelle2006}. Most of these models can be written as
\begin{align}
\label{eq:sslmodel}
    f_{\text{dif}}(\mathcal{L}) =h(\mathbf{A})\mathbf{Y}_{\mathcal{L}}
\end{align}
where  $h$ is e.g. a polynomial function, and
$\mathbf{Y}_{\mathcal{L}}$ is an $N\times C$ matrix whose $n$-th row is $\mathbf{y}_n$ if $n\in\mathcal{L}$, and $\mathbf{0}_{C}$ otherwise. A common choice is $h(\mathbf{A})= \sum_{t=0}^T \alpha_t (\mathbf{D}^{-1/2}\mathbf{A}\mathbf{D}^{-1/2})^t$, where $\alpha_t >0$ and $\mathbf{D}=\text{diag}(\mathbf{A}\mathbf{1})$ denotes the degree matrix. 
\begin{corollary} (Proof in Section C of the supplementary material.) 
Let $f_{\text{dif}}(\mathcal{L})=h(\mathbf{A})\mathbf{Y}_{\mathcal{L}}$ with $h(\mathbf{A})=[\mathbf{h}_1,\ldots,\mathbf{h}_N]$ be the diffusion matrix, and $\mathcal{N}:=
\mathcal{V}-\mathcal{A}$ the set containing the nominal nodes.
The total variation distance between $\mathbf{P}_{G}$ and $\mathbf{P}_{N}$  is 
\begin{align}
    \label{eq:finres1}
     \|\mathbf{P}_G
- \mathbf{P}_{N} \|_1=
p_{fa}\tfrac{f_{\mathcal{A}}}{|\bar{\mathcal{L}}^c_S|}
\bigg{\|}
\tfrac{K}{N-K}\sum_{n\in\mathcal{N}}
    \mathbf{h}_n\mathbf{y}_n^\transpose -\sum_{n'\in\mathcal{A}}
    \mathbf{h}_{n'}\mathbf{y}_{n'}^\transpose \bigg{\|}_1
\end{align}
where $f_{\mathcal{A}}:={N-1 \choose S-1}$  and $|\bar{\mathcal{L}}^c_S|={N\choose S} -{N-K\choose S}$.
\end{corollary}
Upon applying the reverse triangle inequality to \eqref{eq:finres1}, Corollary 1 yields  
\begin{align}
     \|\mathbf{P}_N
- \mathbf{P}_{A} \|_1\ge&
\tfrac{f_{\mathcal{A}}}{|\bar{\mathcal{L}}^c_S|}
\bigg|\|\sum_{{n'}\in\mathcal{A}}
    \mathbf{h}_{n'}\mathbf{y}_{n'}^\transpose \|_1-
\tfrac{K}{N-K}\|\sum_{n\in\mathcal{N}}\mathbf{h}_n\mathbf{y}_n^\transpose\|_1\bigg|\nonumber\\
=&\tfrac{f_{\mathcal{A}}}{|\bar{\mathcal{L}}^c_S|}
\bigg|\sum_{{n'}\in\mathcal{A}}
    \|\mathbf{h}_{n'}\|_1-
\tfrac{K}{N-K}\sum_{n\in\mathcal{N}}\|\mathbf{h}_n\|_1\bigg|\label{eq:finineq}
\end{align}
where \eqref{eq:finineq} follows since $\mathbf{y}_n$ has entries either 0 or 1, and $\mathbf{h}_n$ has nonegative entries. Hence, anomalous nodes with large $\|\mathbf{h}_n\|_1$ contribute more to the error norm $\|\mathbf{P}_N-\mathbf{P}_{G} \|_1$. In diffusion-based models, nodes with larger $\|\mathbf{h}_n\|_1$ typically have higher degree. Consequently, as expected anomalous nodes with high degree have a large effect on the learned distributions.

\section{Experiments}
In this section, we compare GraphSAC with state-of-the-art alternatives under different anomaly generation models based on random walks, clustered anomalies, as well as contemporary adversarial attacks for graph data.
\begin{table}[]
    \centering
    \caption{AUC values for detecting adversarial attacks.} 
    \rowcolors[]{1}{white}{gray}
    \begin{tabular}{p{1.4cm} c c c c}
    		\toprule
\texttt{Dataset} & \texttt{Citeseer} & \texttt{Polblog} & \texttt{Cora} &\texttt{Pubmed}\\
   \midrule 
    GraphSAC & \textbf{0.75} &\textbf{0.98}
&\textbf{0.80} &\textbf{0.82}
 \\
    Gae  & 0.64 &0.51 &
0.50 &0.69
\\
    
    Amen & 0.73 &0.89 &
0.75& 0.62
\\
    
    Radar& 0.67&
0.76& 0.77 &0.44
\\
    Degree & 0.58&
0.48& 0.40&0.57
\\
 Cut ratio & 0.49&
0.51&0.35&0.55
\\
Flake & 0.47&
0.61&0.46&0.60
\\
Conductance & 0.35&
0.39& 0.61&0.59
\\
		\bottomrule
    \end{tabular}
    \label{tab:adv}
\end{table}
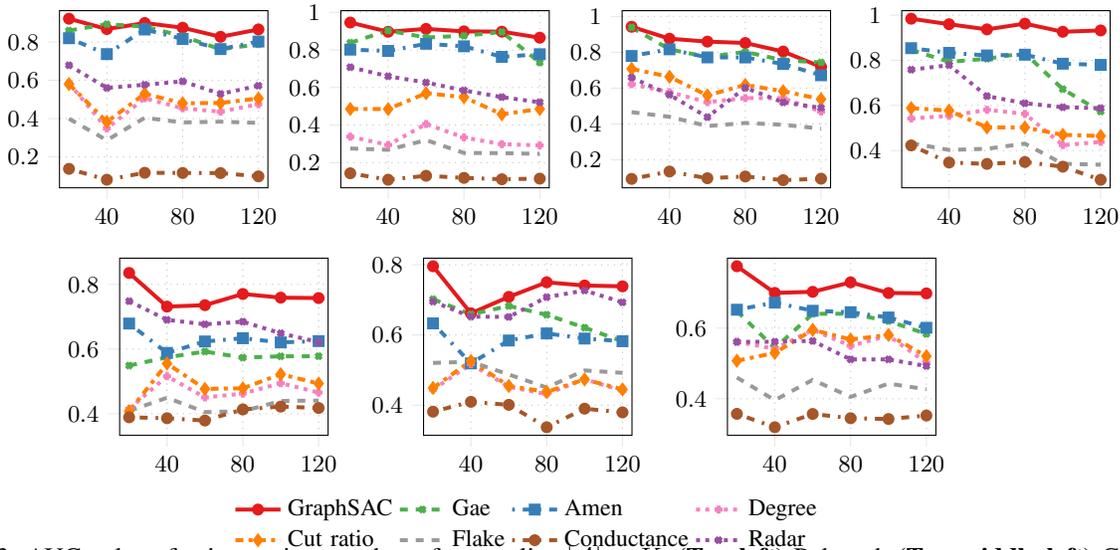
\begin{figure*}[t]
\hspace{1.5cm}
\begin{tikzpicture}

\definecolor{color0}{rgb}{1,0.647058823529412,0}
\definecolor{color1}{rgb}{1,1,0}
\definecolor{color2}{rgb}{0.501960784313725,0,0.501960784313725}

\begin{axis}[width=0.8\mywidth,
height=0.987\myheight,
at={(0\mywidth,0\myheight)},
xtick={40,80,120},
legend cell align={left},
legend style={at={(0.5,0.09)}, anchor=south, draw=white!80.0!black},
tick align=outside,
tick pos=left,grid style={dotted},
xmin=15, xmax=125,
xmajorgrids,
ymajorgrids,ticklabel style={font=\small},
ymin=0.038839533836649, ymax=0.963332207731024
]

\addplot [mark size=\markwidth, line width=\mylinewidth, GraphSAC, mark=*,  mark options={solid}]
table [row sep=\\]{%
20	0.921309813463098 \\
40	0.86735089869281 \\
60	0.898778143876991 \\
80	0.874233001658375 \\
100	0.826955796497081 \\
120	0.865281285072952 \\
};
\addplot [mark size=\markwidth, line width=\mylinewidth, Gae, dashed, mark=x,  mark options={solid}]
table [row sep=\\]{%
20	0.8594768856447689 \\
40	0.892514297385621 \\
60	0.880299286106535 \\
80	0.830182421227197 \\
100	0.761576313594662 \\
120	0.788173400673401 \\
};
\addplot [mark size=\markwidth, line width=\mylinewidth, Amen, dash pattern=on 1pt off 3pt on 3pt off 3pt, mark=square*,  mark options={solid}]
table [row sep=\\]{%
20	0.81948499594485 \\
40	0.736580882352941 \\
60	0.86499176276771 \\
80	0.815573175787728 \\
100	0.76278565471226 \\
120	0.802076318742985 \\
};
\addplot [mark size=\markwidth, line width=\mylinewidth, Avg, dotted, mark=+,  mark options={solid}]
table [row sep=\\]{%
20	0.584772911597729 \\
40	0.348549836601307 \\
60	0.506308347062054 \\
80	0.453560323383085 \\
100	0.436092577147623 \\
120	0.474824635241302 \\
};
\addplot [mark size=\markwidth, line width=\mylinewidth, Cut, dashdotted, mark=diamond*,  mark options={solid}]
table [row sep=\\]{%
20	0.580089213300892 \\
40	0.382843137254902 \\
60	0.529660900604064 \\
80	0.479840381426202 \\
100	0.481743119266055 \\
120	0.505183782267116 \\
};
\addplot [mark size=\markwidth, line width=\mylinewidth, Flake, dashed]
table [row sep=\\]{%
20	0.399391727493917 \\
40	0.287122140522876 \\
60	0.404159802306425 \\
80	0.37988184079602 \\
100	0.38394495412844 \\
120	0.377402497194164 \\
};
\addplot [mark size=\markwidth, line width=\mylinewidth, Conductance, dash pattern=on 1pt off 3pt on 3pt off 3pt, mark=*,  mark options={solid}]
table [row sep=\\]{%
20	0.137631792376318 \\
40	0.0808619281045752 \\
60	0.11722954420648 \\
80	0.116148424543947 \\
100	0.115650542118432 \\
120	0.0984041806958474 \\
};
\addplot [mark size=\markwidth, line width=\mylinewidth, Radar, dotted, mark=x,  mark options={solid}]
table [row sep=\\]{%
20	0.67816301703163 \\
40	0.560028594771242 \\
60	0.57664744645799 \\
80	0.594574004975124 \\
100	0.528865721434529 \\
120	0.571447109988777 \\
};
\end{axis}
\pgfresetboundingbox
\useasboundingbox ($(current axis.south west)+( \marginl, \margin ex)$)
    rectangle (current axis.north east);
\end{tikzpicture}
\hspace{0.7cm}
\begin{tikzpicture}

\definecolor{color0}{rgb}{1,0.647058823529412,0}
\definecolor{color1}{rgb}{1,1,0}
\definecolor{color2}{rgb}{0.501960784313725,0,0.501960784313725}

\begin{axis}[
width=0.8\mywidth,
height=0.987\myheight,
at={(0\mywidth,0\myheight)},
legend cell align={left},
legend style={at={(0.91,0.5)}, anchor=east, draw=white!80.0!black},
tick align=outside,
tick pos=left,grid style={dotted},
xtick={40,80,120},
xmin=15, xmax=125,
xmajorgrids,
ymajorgrids,
ymin=0.0659894518002599,ticklabel style={font=\small}, ymax=1.00926648856289
]
\addlegendimage{no markers, red}
\addlegendimage{no markers, green!50.19607843137255!black}
\addlegendimage{no markers, blue}
\addlegendimage{no markers, color0}
\addlegendimage{no markers, black}
\addlegendimage{no markers, color1}
\addlegendimage{no markers, color2}
\addplot [mark size=\markwidth, line width=\mylinewidth, GraphSAC, mark=*,  mark options={solid}]
table [row sep=\\]{%
20	0.946082652180535 \\
40	0.896937033084312 \\
60	0.912183106950874 \\
80	0.899079590589673 \\
100	0.897362895005097 \\
120	0.86540731292517 \\
};
\addplot [mark size=\markwidth, line width=\mylinewidth, Gae, dashed, mark=x,  mark options={solid}]
table [row sep=\\]{%
20	0.838554094532162 \\
40	0.903186461350816 \\
60	0.866390259619135 \\
80	0.875660326917201 \\
100	0.895974006116208 \\
120	0.733113945578231 \\
};
\addplot [mark size=\markwidth, line width=\mylinewidth, Amen, dash pattern=on 1pt off 3pt on 3pt off 3pt, mark=square*,  mark options={solid}]
table [row sep=\\]{%
20	0.801677920495507 \\
40	0.794903948772679 \\
60	0.832956877109088 \\
80	0.81949791221102 \\
100	0.763238022426096 \\
120	0.777360969387755 \\
};
\addplot [mark size=\markwidth, line width=\mylinewidth, Avg, dotted, mark=+,  mark options={solid}]
table [row sep=\\]{%
20	0.338353556379144 \\
40	0.294167047822331 \\
60	0.405487442980448 \\
80	0.334009955188919 \\
100	0.298185524974516 \\
120	0.291919642857143 \\
};
\addplot [mark size=\markwidth, line width=\mylinewidth, Cut, dashdotted, mark=diamond*,  mark options={solid}]
table [row sep=\\]{%
20	0.485685637406712 \\
40	0.485269604106317 \\
60	0.569639314239202 \\
80	0.548370506161524 \\
100	0.457600407747197 \\
120	0.485476615646259 \\
};
\addplot [mark size=\markwidth, line width=\mylinewidth, Flake, dashed]
table [row sep=\\]{%
20	0.275305884144794 \\
40	0.268209076586878 \\
60	0.319897068050398 \\
80	0.251542926978307 \\
100	0.250425076452599 \\
120	0.246741071428571 \\
};
\addplot [mark size=\markwidth, line width=\mylinewidth, Conductance, dash pattern=on 1pt off 3pt on 3pt off 3pt, mark=*,  mark options={solid}]
table [row sep=\\]{%
20	0.143196933543179 \\
40	0.108865680744016 \\
60	0.129918943209374 \\
80	0.118761457378552 \\
100	0.111438837920489 \\
120	0.114783588435374 \\
};
\addplot [mark size=\markwidth, line width=\mylinewidth, Radar, dotted, mark=x,  mark options={solid}]
table [row sep=\\]{%
20	0.70816301703163 \\
40	0.660028594771242 \\
60	0.62664744645799 \\
80	0.584574004975124 \\
100	0.548865721434529 \\
120	0.521447109988777 \\
};
\end{axis}
\pgfresetboundingbox
\useasboundingbox ($(current axis.south west)+(\marginl, \margin ex)$)
    rectangle (current axis.north east);
\end{tikzpicture}
\hspace{0.7cm}
\begin{tikzpicture}

\definecolor{color0}{rgb}{1,0.647058823529412,0}
\definecolor{color1}{rgb}{1,1,0}
\definecolor{color2}{rgb}{0.501960784313725,0,0.501960784313725}

\begin{axis}[width=0.8\mywidth,
height=0.987\myheight,
at={(0\mywidth,0\myheight)},grid style={dotted},
legend cell align={left},
legend style={at={(0.03,0.03)}, anchor=south west, draw=white!80.0!black},
tick align=outside,
tick pos=left,
xtick={40,80,120},
xmin=15, xmax=125,
ymin=0.0443697454844007, xmajorgrids,ticklabel style={font=\small},
ymajorgrids,
legend columns=4,
legend style={
	at={(0,1.015)}, 
	anchor=south west, legend cell align=left, align=left, draw=none
	, font=\legendfontsize}
]
\addplot [mark size=\markwidth, line width=\mylinewidth, GraphSAC, mark=*,  mark options={solid}]
table [row sep=\\]{%
20	0.94238095238095 \\
40	0.875845341018252 \\
60	0.859695357833656 \\
80	0.851948924731183 \\
100	0.804320197044335 \\
120	0.718526307108692 \\
};
\addplot [mark size=\markwidth, line width=\mylinewidth, Gae, dashed, mark=x,  mark options={solid}]
table [row sep=\\]{%
20	0.938880952380952 \\
40	0.816402497598463 \\
60	0.774036105738233 \\
80	0.802337487781036 \\
100	0.755566502463054 \\
120	0.743509813623619 \\
};
\addplot [mark size=\markwidth, line width=\mylinewidth, Amen, dash pattern=on 1pt off 3pt on 3pt off 3pt, mark=square*,  mark options={solid}]
table [row sep=\\]{%
20	0.7795 \\
40	0.816606628242075 \\
60	0.77153449387492 \\
80	0.771230449657869 \\
100	0.734901477832512 \\
120	0.671911594920007 \\
};
\addplot [mark size=\markwidth, line width=\mylinewidth, Avg, dotted, mark=+,  mark options={solid}]
table [row sep=\\]{%
20	0.622142857142857 \\
40	0.577917867435159 \\
60	0.521921341070277 \\
80	0.54469696969697 \\
100	0.543807881773399 \\
120	0.469157182912749 \\
};
\addplot [mark size=\markwidth, line width=\mylinewidth, Cut, dashdotted, mark=diamond*,  mark options={solid}]
table [row sep=\\]{%
20	0.70752380952381 \\
40	0.663460614793468 \\
60	0.559784010315925 \\
80	0.618212365591398 \\
100	0.583024630541872 \\
120	0.538392709879598 \\
};
\addplot [mark size=\markwidth, line width=\mylinewidth, Flake, dashed]
table [row sep=\\]{%
20	0.466261904761905 \\
40	0.44026176753122 \\
60	0.388733075435203 \\
80	0.405687927663734 \\
100	0.395862068965517 \\
120	0.374307273626917 \\
};
\addplot [mark size=\markwidth, line width=\mylinewidth, Conductance, dash pattern=on 1pt off 3pt on 3pt off 3pt, mark=*,  mark options={solid}]
table [row sep=\\]{%
20	0.0939285714285714 \\
40	0.135038424591739 \\
60	0.0976950354609929 \\
80	0.107258064516129 \\
100	0.0869655172413793 \\
120	0.0951179284182748 \\
};
\addplot [mark size=\markwidth, line width=\mylinewidth, Radar, dotted, mark=x,  mark options={solid}]
table [row sep=\\]{%
20	0.65947619047619 \\
40	0.563436599423631 \\
60	0.438483236621534 \\
80	0.602260508308895 \\
100	0.520793103448276 \\
120	0.490545109681676 \\
};
\end{axis}
\pgfresetboundingbox
\useasboundingbox ($(current axis.south west)+( \marginl,\margin ex)$)
    rectangle (current axis.north east);
\end{tikzpicture}
\hspace{0.8cm}
\begin{tikzpicture}

\definecolor{color0}{rgb}{1,0.647058823529412,0}
\definecolor{color1}{rgb}{1,1,0}
\definecolor{color2}{rgb}{0.501960784313725,0,0.501960784313725}

\begin{axis}[width=0.8\mywidth,
height=0.987\myheight,
at={(0\mywidth,0\myheight)},
tick align=outside,
tick pos=left,
grid style={dotted},
xmin=15, xmax=125,
xtick={40,80,120},ticklabel style={font=\small},
xmajorgrids,
ymajorgrids,
ymin=0.235012802068754, ymax=1.02008615053931,
legend columns=4,
legend style={
	at={(0,1.015)}, 
	anchor=south west, legend cell align=left, align=left, draw=none
	, font=\legendfontsize}
]
\addplot [mark size=\markwidth, line width=\mylinewidth, GraphSAC, mark=*,  mark options={solid}]
table [row sep=\\]{%
20	0.984400998336107 \\
40	0.959932375316991 \\
60	0.937138769670959 \\
80	0.962510926573427 \\
100	0.926530612244898 \\
120	0.932460890493381 \\
};
\addplot [mark size=\markwidth, line width=\mylinewidth, Gae, dashed, mark=x,  mark options={solid}]
table [row sep=\\]{%
20	0.846555740432612 \\
40	0.79302620456467 \\
60	0.807138769670959 \\
80	0.831899038461539 \\
100	0.672200532386868 \\
120	0.57314229843562 \\
};
\addplot [mark size=\markwidth, line width=\mylinewidth, Amen, dash pattern=on 1pt off 3pt on 3pt off 3pt, mark=square*,  mark options={solid}]
table [row sep=\\]{%
20	0.853818635607321 \\
40	0.833136094674556 \\
60	0.82173104434907 \\
80	0.824868881118881 \\
100	0.784977817213842 \\
120	0.779836040914561 \\
};
\addplot [mark size=\markwidth, line width=\mylinewidth, Avg, dotted, mark=+,  mark options={solid}]
table [row sep=\\]{%
20	0.542221297836938 \\
40	0.55363482671175 \\
60	0.580472103004292 \\
80	0.562871503496504 \\
100	0.424046140195209 \\
120	0.438334837545126 \\
};
\addplot [mark size=\markwidth, line width=\mylinewidth, Cut, dashdotted, mark=diamond*,  mark options={solid}]
table [row sep=\\]{%
20	0.588144758735441 \\
40	0.576901944209636 \\
60	0.50241773962804 \\
80	0.502408216783217 \\
100	0.469618456078083 \\
120	0.464538206979543 \\
};
\addplot  [mark size=\markwidth, line width=\mylinewidth, Flake, dashed]
table [row sep=\\]{%
20	0.432861896838602 \\
40	0.40259932375317 \\
60	0.408011444921316 \\
80	0.430594405594406 \\
100	0.341330967169476 \\
120	0.337379663056558 \\
};
\addplot [mark size=\markwidth, line width=\mylinewidth, Conductance, dash pattern=on 1pt off 3pt on 3pt off 3pt, mark=*,  mark options={solid}]
table [row sep=\\]{%
20	0.42283693843594 \\
40	0.346999154691462 \\
60	0.340886981402003 \\
80	0.349147727272727 \\
100	0.328473824312334 \\
120	0.270697954271962 \\
};
\addplot [mark size=\markwidth, line width=\mylinewidth, Radar, dotted, mark=x,  mark options={solid}]
table [row sep=\\]{%
20	0.758552412645591 \\
40	0.778550295857988 \\
60	0.64134477825465 \\
80	0.609709353146853 \\
100	0.592200532386868 \\
120	0.58789861612515 \\
};
\end{axis}
\pgfresetboundingbox
\useasboundingbox ($(current axis.south west)+( \marginl, \margin ex)$)
    rectangle (current axis.north east);
\end{tikzpicture}

\vspace{1cm}\hspace{2.3cm}
\begin{tikzpicture}

\definecolor{color0}{rgb}{1,0.647058823529412,0}
\definecolor{color1}{rgb}{1,1,0}
\definecolor{color2}{rgb}{0.501960784313725,0,0.501960784313725}

\begin{axis}[width=0.8\mywidth,
height=0.987\myheight,
at={(0\mywidth,0\myheight)},
legend style={draw=white!80.0!black},
tick align=outside,
tick pos=left,
x grid style={white!69.01960784313725!black},
xmin=15, xmax=125,
xtick={40,80,120},
xmajorgrids,
ymajorgrids,ticklabel style={font=\small},
grid style={dotted}]
\addlegendimage{no markers, GraphSAC}
\addlegendimage{no markers,Gae}
\addlegendimage{no markers, Amen}
\addlegendimage{no markers, color0}
\addlegendimage{no markers, black}
\addlegendimage{no markers, color1}
\addlegendimage{no markers, color2}
\addplot [line width=\mylinewidth, GraphSAC, mark=*, mark size=\markwidth, mark options={solid}]
table [row sep=\\]{%
20	0.835347842984843 \\
40	0.731296242211838 \\
60	0.735687115966059 \\
80	0.770378469507428 \\
100	0.758977673325499 \\
120	0.757670296608784 \\
};
\addplot [mark size=\markwidth, line width=\mylinewidth, Gae, dashed, mark=x,  mark options={solid}]
table [row sep=\\]{%
20	0.548955499417023 \\
40	0.574435358255452 \\
60	0.592241295230664 \\
80	0.573508356137608 \\
100	0.5775910693302 \\
120	0.577784590732202 \\
};
\addplot [mark size=\markwidth, line width=\mylinewidth, Amen, dash pattern=on 1pt off 3pt on 3pt off 3pt, mark=square*,  mark options={solid}]
table [row sep=\\]{%
20	0.678896230081617 \\
40	0.58808167834891 \\
60	0.623498813355441 \\
80	0.633448983580923 \\
100	0.620630630630631 \\
120	0.624601850943458 \\
};
\addplot [mark size=\markwidth, line width=\mylinewidth, Avg, dotted, mark=+,  mark options={solid}]
table [row sep=\\]{%
20	0.415225417800233 \\
40	0.515960864485981 \\
60	0.450523424038493 \\
80	0.462242230258014 \\
100	0.494510379945163 \\
120	0.465598122894797 \\
};
\addplot [mark size=\markwidth, line width=\mylinewidth, Cut, dashdotted, mark=diamond*,  mark options={solid}]
table [row sep=\\]{%
20	0.407340652934318 \\
40	0.555257009345794 \\
60	0.476787281771189 \\
80	0.478600224784988 \\
100	0.522335487661575 \\
120	0.49340233493574 \\
};
\addplot [mark size=\markwidth, line width=\mylinewidth, Flake, dashed]
table [row sep=\\]{%
20	0.408715507190051 \\
40	0.449549746884735 \\
60	0.40513995903638 \\
80	0.407693999218139 \\
100	0.439522130826479 \\
120	0.440753458255666 \\
};
\addplot [mark size=\markwidth, line width=\mylinewidth, Conductance, dash pattern=on 1pt off 3pt on 3pt off 3pt, mark=*,  mark options={solid}]
table [row sep=\\]{%
20	0.389035172949864 \\
40	0.386129770249221 \\
60	0.378988263597646 \\
80	0.413375928459734 \\
100	0.422128867998433 \\
120	0.417816965891625 \\
};

\addplot [mark size=\markwidth, line width=\mylinewidth, Radar, dotted, mark=x,  mark options={solid}]
table [row sep=\\]{%
20	0.74816301703163 \\
40	0.69028594771242 \\
60	0.67664744645799 \\
80	0.684574004975124 \\
100	0.648865721434529 \\
120	0.621447109988777 \\
};
\end{axis}
\pgfresetboundingbox
\useasboundingbox ($(current axis.south west)+( \marginl,\margin ex)$)
    rectangle (current axis.north east);
\end{tikzpicture}
\hspace{1cm}
\begin{tikzpicture}

\definecolor{color0}{rgb}{1,0.647058823529412,0}
\definecolor{color1}{rgb}{1,1,0}
\definecolor{color2}{rgb}{0.501960784313725,0,0.501960784313725}

\begin{axis}[width=0.8\mywidth,
height=0.987\myheight,
at={(0\mywidth,0\myheight)},
xtick={40,80,120},
legend cell align={left},
legend columns=4,ticklabel style={font=\small},
legend style={
	at={(0.5,-0.3)},anchor=north, legend cell align=left, align=left, draw=none
	, font=\legendfontsize}
,
tick align=outside,
tick pos=left,grid style={dotted},
xmin=15, xmax=125,
xmajorgrids,
ymajorgrids,
legend entries={{GraphSAC},{Gae},{Amen},{Degree},{Cut ratio},{Flake},{Conductance},{Radar}},
ymin=0.314232279058563, ymax=0.818619454485536
]
\addplot [mark size=\markwidth, line width=\mylinewidth, GraphSAC, mark=*,  mark options={solid}]
table [row sep=\\]{%
20	0.795692764693401 \\
40	0.661599914784832 \\
60	0.708533019540722 \\
80	0.749965091299678 \\
100	0.740882797323549 \\
120	0.738450292397661 \\
};
\addplot [mark size=\markwidth, line width=\mylinewidth, Gae, dashed, mark=x,  mark options={solid}]
table [row sep=\\]{%
20	0.701888393804371 \\
40	0.658819769919046 \\
60	0.682481101126801 \\
80	0.65750268528464 \\
100	0.620748974746385 \\
120	0.581799148075951 \\
};
\addplot [mark size=\markwidth, line width=\mylinewidth, Amen, dash pattern=on 1pt off 3pt on 3pt off 3pt, mark=square*,  mark options={solid}]
table [row sep=\\]{%
20	0.633004455760662 \\
40	0.519029612270984 \\
60	0.584239052916845 \\
80	0.6043716433942 \\
100	0.589764731275631 \\
120	0.58209154573677 \\
};
\addplot [mark size=\markwidth, line width=\mylinewidth, Avg, dotted, mark=+,  mark options={solid}]
table [row sep=\\]{%
20	0.444313600678973 \\
40	0.516947166595654 \\
60	0.452242904007988 \\
80	0.431452738990333 \\
100	0.474278005611915 \\
120	0.440780449065049 \\
};
\addplot [mark size=\markwidth, line width=\mylinewidth, Cut, dashdotted, mark=diamond*,  mark options={solid}]
table [row sep=\\]{%
20	0.449087629959686 \\
40	0.525974648487431 \\
60	0.454714020824419 \\
80	0.4378141783029 \\
100	0.473369307144399 \\
120	0.444998556060934 \\
};
\addplot [mark size=\markwidth, line width=\mylinewidth, Flake, dashed]
table [row sep=\\]{%
20	0.520411627413537 \\
40	0.523679164891351 \\
60	0.487412637284268 \\
80	0.450209452201933 \\
100	0.498981221670624 \\
120	0.492061944985922 \\
};
\addplot [mark size=\markwidth, line width=\mylinewidth, Conductance, dash pattern=on 1pt off 3pt on 3pt off 3pt, mark=*,  mark options={solid}]
table [row sep=\\]{%
20	0.381476766390834 \\
40	0.409485513421389 \\
60	0.400962772785623 \\
80	0.337158968850698 \\
100	0.390062594431254 \\
120	0.379487040646885 \\
};
\addplot [mark size=\markwidth, line width=\mylinewidth, Radar, dotted, mark=x,  mark options={solid}]
table [row sep=\\]{%
20	0.694971355824316 \\
40	0.65219429058372 \\
60	0.65132648694908 \\
80	0.707696025778733 \\
100	0.726647960284913 \\
120	0.692877770558082 \\
};
\end{axis}
\pgfresetboundingbox
\useasboundingbox ($(current axis.south west)+(\marginl, \margin ex)$)
    rectangle (current axis.north east);
\end{tikzpicture}
\hspace{1cm}
\begin{tikzpicture}

\definecolor{color0}{rgb}{1,0.647058823529412,0}
\definecolor{color1}{rgb}{1,1,0}
\definecolor{color2}{rgb}{0.501960784313725,0,0.501960784313725}

\begin{axis}[
width=0.8\mywidth,
height=0.987\myheight,
at={(0\mywidth,0\myheight)},
legend cell align={left},
legend style={at={(0.91,0.5)}, anchor=east, draw=white!80.0!black},
tick align=outside,ticklabel style={font=\small},
tick pos=left,grid style={dotted},
xtick={40,80,120},
xmin=15, xmax=125,
xmajorgrids,
ymajorgrids,
ymin=0.296631638063005, ymax=0.797984486068975
]

\addplot [mark size=\markwidth, line width=\mylinewidth, GraphSAC, mark=*,  mark options={solid}]
table [row sep=\\]{%
20	0.775195720250522 \\
40	0.699725937581956 \\
60	0.702781439493804 \\
80	0.729555393586006 \\
100	0.699515054622968 \\
120	0.698331548893647 \\
};
\addplot [mark size=\markwidth, line width=\mylinewidth, Gae, dashed, mark=x,  mark options={solid}]
table [row sep=\\]{%
20	0.64830375782881 \\
40	0.543214004720692 \\
60	0.640447315229809 \\
80	0.640932282003711 \\
100	0.620151878497202 \\
120	0.583485010706638 \\
};
\addplot [mark size=\markwidth, line width=\mylinewidth, Amen, dash pattern=on 1pt off 3pt on 3pt off 3pt, mark=square*,  mark options={solid}]
table [row sep=\\]{%
20	0.652100730688935 \\
40	0.672233149750852 \\
60	0.649178310923631 \\
80	0.645020540683806 \\
100	0.630130562216893 \\
120	0.600912294789436 \\
};
\addplot [mark size=\markwidth, line width=\mylinewidth, Avg, dotted, mark=+,  mark options={solid}]
table [row sep=\\]{%
20	0.56143006263048 \\
40	0.541306058221872 \\
60	0.594955620001758 \\
80	0.549847601378214 \\
100	0.578824940047962 \\
120	0.502105638829408 \\
};
\addplot [mark size=\markwidth, line width=\mylinewidth, Cut, dashdotted, mark=diamond*,  mark options={solid}]
table [row sep=\\]{%
20	0.507625260960334 \\
40	0.529727248885392 \\
60	0.595184111081817 \\
80	0.56786045587066 \\
100	0.580724753530509 \\
120	0.520543361884368 \\
};
\addplot [mark size=\markwidth, line width=\mylinewidth, Flake, dashed]
table [row sep=\\]{%
20	0.460294885177453 \\
40	0.396000524521374 \\
60	0.453510853326303 \\
80	0.405386959978797 \\
100	0.442933653077538 \\
120	0.426911581013562 \\
};
\addplot [mark size=\markwidth, line width=\mylinewidth, Conductance, dash pattern=on 1pt off 3pt on 3pt off 3pt, mark=*,  mark options={solid}]
table [row sep=\\]{%
20	0.357346033402923 \\
40	0.319420403881458 \\
60	0.357043676948765 \\
80	0.344950967399947 \\
100	0.342749800159872 \\
120	0.352600820842256 \\
};
\addplot [mark size=\markwidth, line width=\mylinewidth, Radar, dotted, mark=x,  mark options={solid}]
table [row sep=\\]{%
20	0.561007306889353 \\
40	0.561498819826908 \\
60	0.563727041040513 \\
80	0.511423270606944 \\
100	0.511502797761791 \\
120	0.492643647394718 \\
};
\end{axis}
\pgfresetboundingbox
\useasboundingbox ($(current axis.south west)+( \marginl, \margin ex)$)
    rectangle (current axis.north east);
\end{tikzpicture}

\vspace{1.3cm}
    \caption{AUC values for increasing number of anomalies $|\mathcal{A}|=K$. \textbf{(Top left)} Pubmed,
    \textbf{(Top middle left)} Cora, \textbf{(Top middle right)} Citeseer, \textbf{(Top right)} Polblogs,
    \textbf{(Bottom left)} Blogcat, \textbf{(Bottom middle)} Wikipedia \textbf{(Bottom right)} PPI. 
    }
    \label{fig:labpert}
\end{figure*}

\noindent\textbf{Baselines}. Amen identifies anomalies by evaluating the attribute correlation of nodes per egonet of the graph~\cite{perozzi2016scalable}. The graph neural network encoder (GAE) ranks anomalies by measuring the reconstruction error recovered at the decoder 
output~\cite{ding2019deep}. Radar asserts nodes as anomalous if they do not adhere to the proposed parametric model~\cite{li2017radar}. Following the experimental setup in \cite{perozzi2016scalable}, additional anomaly detection methods are considered that only utilize graph connectivity but not nodal labels, namely approaches based on the Average degree~\cite{charikar2000greedy}, Cut ratio~\cite{fortunato2010community}, Flake~\cite{flake2000efficient}, and Conductance~\cite{andersen2006local}. 
Unless stated otherwise, GraphSAC  is configured with $T=0.5$, $I=50$ and the personalized PageRank (PPR)~\cite{bahmani2010fast} as the SSL model.
The different methods are evaluated using the area under the curve (AUC) of the receiver operating characteristic (ROC) curve. The ROC curve  plots the rate an anomaly is detected (true positive) against the rate a node is miss-classified as anomalous (false positive). The AUC value represents the probability that a randomly chosen abnormal node is ranked higher than a normal node.

\noindent\textbf{Datasets}. The 7 benchmark labeled graphs are Cora ($N=2708,C=7$), Citeseer ($N=3327, C=6$), Pubmed ($N=19717,C=3$), Polblogs ($N=1224,C=2$), Blogcat ($N=10312,C=39$), PPI ($N=3890,C=50$), and Wikipedia ($N=4733, C=39$). The nodes in the last three graphs are multilabel ones.\footnote{For graphs with multilabeled nodes, clustered anomalies and adversarial attacks are not defined and hence, these graphs are not included in the respecitive experiments.}  

\subsection{Adversarial attacks}
We generated anomalies using the adversarial setup in~\cite{zugner18adv}, where attacks are effected on attributed graphs targeted for GCNs. We focus on structural attacks, which means that edges adjacent to the targeted node are added or removed; that is, we select a random subset of targeted nodes $\mathcal{A}$, and alter their connectivity by a sequence of structural attacks~\cite{zugner18adv}. 

Table \ref{tab:adv} reports the AUC values for competing state-of-the-art techniques in detecting adversarial attacks with $K$=10 targeted nodes. As GAE relies on a deep graph autonencoder~\cite{ding2019deep}, it is maximally affected by the adversarial attacks. Our novel method outperforms all alternatives in detecting the  attacked nodes. These promising results suggest that GraphSAC can be effectively employed as a preprocessing step to flag adversarial input to a graph neural network.

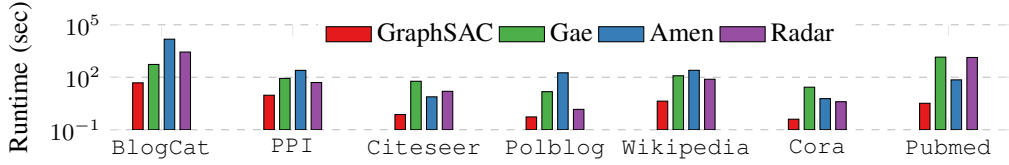
\begin{figure*}[t]
\hspace{2cm}
%
%
%
\begin{tikzpicture}

\begin{axis}[%
axis line style={draw=none},
width=2.2\mywidth,
height=0.35\myheight,
at={(0\mywidth,0\myheight)},
scale only axis,
bar shift auto,
log origin=infty,
xmin=0.509090909090909,ticklabel style={font=\small},
xmax=9.090909090909,
xtick={1,2.25,3.5,4.75,6,7.25,8.5},
xticklabels={{\texttt{BlogCat}},{\texttt{PPI}},{\texttt{Citeseer}},{\texttt{Polblog}}, {\texttt{Wikipedia}},{\texttt{Cora}},{\texttt{Pubmed}}},
ymode=log,
ymin=0.1,
ymax=100000,
yminorticks=true,label style={font=\small},
ylabel={Runtime (sec) },
ylabel near ticks,
axis background/.style={fill=white},
ymajorgrids,
 legend columns=4,
legend style={at={(0.23,0.70)}, anchor=south west, legend cell align=left, align=left, draw=none},grid style={dashed}
]

\addplot[ybar, bar width=0.1, fill=GraphSAC, draw=black, area legend] table[row sep=crcr] {%
	1	48\\
	2.25 9.39 \\
	3.5 0.74\\
	4.75 0.54 \\
	6 4.32\\
	7.25 	0.40\\
	8.5 	3.2497780323028564 \\
};
\addlegendentry{GraphSAC}

\addplot[ybar, bar width=0.1, fill=Gae, draw=black, area legend] table[row sep=crcr] {%
	1	540\\
	2.25 86\\
	3.5 59\\
	4.75 15\\
	6 122\\
	7.25 27\\
	8.5 1412.8371877670288\\
};
\addlegendentry{Gae}

\addplot[ybar, bar width=0.1, fill=Amen, draw=black, area legend] table[row sep=crcr] {%
1	14993\\
2.25 248\\
3.5 7.6\\
4.75 180\\
6 250\\
7.25 6\\
8.5 71\\
};
\addlegendentry{Amen}
\addplot[ybar, bar width=0.1, fill=Radar, draw=black, area legend] table[row sep=crcr] {%
	1	2761\\
	2.25 50\\
	3.5 15.72\\
	4.75 1.46\\
	6 77\\
	7.25 4\\
	8.5 1353\\
};
\addlegendentry{Radar}

\end{axis}
\end{tikzpicture}%
\vspace{-0.4cm}
	\caption{Runtime comparisons across different graphs for Fig.~\ref{fig:labpert}.}\label{fig:runtimemain}
\end{figure*}
\subsection{Random walk-based anomalies}
We test the algorithms in identifying homophilic random-walk based  anomalies. 
To generate these we select a subset of $|\mathcal{A}|$ nodes at random, and alter their labels. For each $n \in \mathcal{A}$, we perform a random walk of length $k=10$, and replace $y_n$ with the label of the landing node. Hence, we modify the labels of the targeted nodes in $\mathcal{A}$ as prescribed by the random walk model. The resulting nodes violate the homophily property. 
    
Fig. \ref{fig:labpert} plots the AUC values of various methods with increasing $K$ on 6 benchmark graphs.  Evidently, GraphSAC outperforms alternatives while the performance of all methods degrades slightly as $K$ increases. Furthermore, Fig.~\ref{fig:runtimemain} reports the runtime of the best performing algorithms for $K=20$ in Fig.~\ref{fig:labpert} across all graphs. Evidently, GraphSAC is significantly faster than these competing methods. For further details on these  runtime comparisons see Section E of the supplementary material.

\subsection{Clustered anomalies}
Next, we consider indentifying clusters of anomalous nodes Fig.\ref{fig:anom_clust}.  Towards generating clustered anomalies, we select a cluster of connected nodes  $\mathcal{A}$ using \cite{cordasco2010community}, and set all labels as $y_n=c, \forall n \in \mathcal{A}$, where $c$ is the least common label in $\mathcal{A}$. The anomalous nodes satisfy the homophily property inside the corrupted cluster that further challenges their detection.

Table \ref{tab:clust_anom} reports the AUC values of anomaly detection algorithms for identifying clustered anomalies. The performance of competing algorithms is heavily affected since within the affected cluster the anomalous nodes appear as nominal. On the other hand, our novel approach utilizes random sampling and consensus strategies, and markedly outperforms the baseline schemes in identifying the ground-truth anomalies. 
\begin{table}[t!]
    \centering
        \caption{AUC values for discovering clusters of anomalous nodes.} 
    \rowcolors[]{1}{white}{gray}
    \begin{tabular}{p{1.4cm} c c c c}
    \toprule
    \texttt{Dataset} & 
    \texttt{Citeseer}
    & \texttt{Polblog} 
    & \texttt{Cora} 
    & \texttt{Pubmed} 
    \\
    \midrule 
GraphSAC & \textbf{0.88} &\textbf{0.74}
&\textbf{0.97}&\textbf{0.96}
 \\
    Gae  & 0.50 &0.30 &
0.58&0.95
\\
    Amen & 0.51 &0.10 &
0.48&0.42
\\
    Radar & 0.86&
0.40& 0.88&0.88
\\
    Degree & 0.46&
0.10& 0.45&0.29
\\
Cut ratio & 0.49&
0.08&0.40& 0.27
\\
Flake & 0.40&
0.05&0.37&0.33
\\
Conductance & 0.55&
0.55& 0.40&0.37
\\
		\bottomrule
    \end{tabular}
    \label{tab:clust_anom}
\end{table}

\noindent\textbf{Parameter sensitivity}. Sensitivity of GraphSAC to $I,T$ is reported in Fig.~\ref{fig:sens}. GraphSAC's AUC performance is stable around the preselected values. As deduced in the discussion after Theorem 2, the number of iterations $I$ for a reliable anomaly score ranking is relatively small. A large enough ratio of correctly classified nodes $T$, aides GraphSAC in discarding contaminated samples since these will result in a large number of miss-classified nodes.  

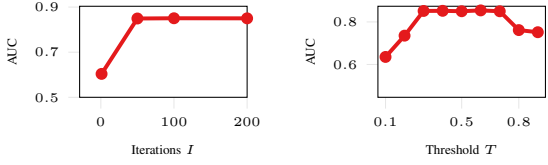
\begin{figure}
    \hspace{-0.2cm}
\begin{tikzpicture}

\begin{axis}[
width=0.7\mywidth,
height=0.7\myheight,
at={(0\mywidth,0\myheight)},grid style={dotted},
legend cell align={left},
legend style={draw=white!80.0!black},
tick align=outside,grid style={dotted},
tick pos=left, ticklabel style={font=\tiny},
x grid style={white!69.01960784313725!black},
xlabel={Iterations $I$},label style={font=\tiny},
xmin=-29, xmax=200,
y grid style={white!69.01960784313725!black},
ylabel={AUC},
ytick={0.5,0.7,0.9},
ymin=0.5, ymax=0.903514479472141]
\addlegendimage{no markers, red}
\addplot [line width=\mylinewidth, GraphSAC, mark=*, mark size=\markwidth, mark options={solid}]
table [row sep=\\]{%
1   0.6038
20	0.850177174975562 \\
50	0.849303519061584 \\
100	0.850384897360704 \\
200	0.85 \\
};
\end{axis}

\end{tikzpicture}~
\begin{tikzpicture}

\begin{axis}[
width=0.7\mywidth,
height=0.7\myheight,
at={(0\mywidth,0\myheight)},grid style={dotted},
legend cell align={left},
legend style={draw=white!80.0!black},
tick align=outside,label style={font=\tiny},
tick pos=left,grid style={dotted},
x grid style={white!69.01960784313725!black},
xlabel={Threshold $T$},
xmin=0.06, xmax=0.94,ticklabel style={font=\tiny},
y grid style={white!69.01960784313725!black},
ylabel={AUC},
xtick={0.1,0.5,0.8},
ytick={0.4,0.6,0.8},
ymin=0.445712060117302, ymax=0.873514479472141
]
\addlegendimage{no markers, red}
\addplot [line width=\mylinewidth, GraphSAC, mark=*, mark size=\markwidth, mark options={solid}]
table [row sep=\\]{%
0.1	0.6354 \\
0.2	0.7354 \\
0.3	0.851447947214076 \\
0.4	0.852052785923754 \\
0.5	0.850103861192571 \\
0.6	0.854068914956012 \\
0.7	0.850006109481916 \\
0.8	0.76157624633431 \\
0.9	0.75157624633431 \\
};
\end{axis}

\end{tikzpicture}         

    \caption{GraphSAC's sensitivity to $I,T$  for the Citeseer dataset with $K=80$ in Fig.~\ref{fig:labpert}.}
    \label{fig:sens} 
\end{figure}

\noindent\textbf{Additional SSL methods}. Fig.~\ref{fig:pol} reports the AUC for  GraphSAC when using the personalized page-rank (ppr)~\cite{bahmani2010fast} and heat-kernel (hk)~\cite{kloster2014heat} SSL methods for the experiment in Fig.~\ref{fig:labpert}. Deep SSL methods~\cite{kipf2016semi} for GraphSAC  will be part of our future research.
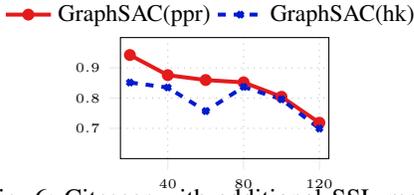
\begin{figure}
   \centering\begin{tikzpicture}

\definecolor{color0}{rgb}{1,0.647058823529412,0}
\definecolor{color1}{rgb}{1,1,0}
\definecolor{color2}{rgb}{0.501960784313725,0,0.501960784313725}

\begin{axis}[width=0.8\mywidth,
height=0.8\myheight,
at={(0\mywidth,0\myheight)},grid style={dotted},
legend cell align={left},
legend entries={{GraphSAC(ppr)},{GraphSAC(hk)}},
legend style={at={(0.03,0.03)}, anchor=south west, draw=white!80.0!black},
tick align=outside,
tick pos=left,
xtick={40,80,120},
ytick={0.7,0.8,0.9},
xmin=15, xmax=125,
ymin=0.6,ymax=1,
xmajorgrids,
ymajorgrids,
legend columns=2,
legend style={
	at={(-0.6,1.015)}, 
	anchor=south west, legend cell align=left, align=left, draw=none
	, font=\small}
]\pgfplotsset{every tick label/.append style={font=\tiny}}
\addplot [mark size=\markwidth, line width=\mylinewidth, GraphSAC, mark=*,  mark options={solid}]
table [row sep=\\]{%
20	0.94238095238095 \\
40	0.875845341018252 \\
60	0.859695357833656 \\
80	0.851948924731183 \\
100	0.804320197044335 \\
120	0.718526307108692 \\
};
\addplot [mark size=\markwidth, line width=\mylinewidth,GraphSAChk, dashed, mark=x, mark options={solid}]
table [row sep=\\]{%
20	0.851452380952381 \\
40	0.834990393852066 \\
60	0.756971308833011 \\
80	0.836986803519062 \\
100	0.795960591133005 \\
120	0.699323767111991 \\
};
\end{axis}
\pgfresetboundingbox
\useasboundingbox ($(current axis.south west)+( \marginl,\margin ex)$)
    rectangle (current axis.north east);
\end{tikzpicture}\vspace{0.2cm}
    \caption{Citeseer with additional SSL methods.}
    \label{fig:pol}
\end{figure}

\noindent\textbf{On the practical interpretation of Theorem 1}. 
Fig. \ref{fig:auc2d} reports the AUC performance of GraphSAC for Citeseer in Fig.~\ref{fig:labpert} with varying number of anomalies $K$ and sampling size $S$, where darker boxes indicate larger values of AUC. Observe that the AUC increases for larger $S$. For larger $K$ the AUC drops as expected, but this can be rectified by increasing $S$. 
Fig. \ref{fig:maxval} demonstrates the maximum number of anomalies $K_m:=\max_i K^{(i)}$ that existed in the contaminated samples that were missclasiffied as nominal $\delta(f(\mathcal{L}^{(i)}))=1$, where $K^{(i)}:=|\mathcal{L}^{(i)}\cup\mathcal{A}|$. Notice $K_m$ in Fig. \ref{fig:maxval} is significantly smaller than the total number of anomalies that is $K=17,33,83$, respectively. 
Hence, even though there are contaminated samples, the number anomalies in $\mathcal{L}$ is so small that does not affect the AUC values in~Fig. \ref{fig:auc2d}. Motivated by this result, we derive  a stronger version of Theorem 1 that accounts for $K_m$ and is included in Section D of the supplementary material.

\begin{figure}[h]
    \centering
    \begin{subfigure}[b]{0.24\textwidth}\centering
        \includegraphics[width=0.6\textwidth]{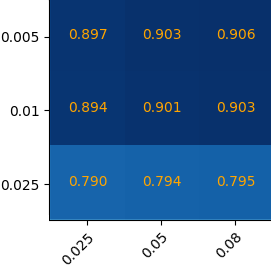}
        \vspace{-0.2cm}
        \caption{AUC values for detecting the anomalous nodes.}
        \label{fig:auc2d}
    \end{subfigure}
    \begin{subfigure}[b]{0.24\textwidth}\centering
        \includegraphics[width=0.6\textwidth]{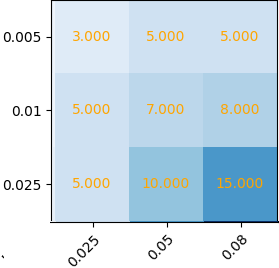}
                \vspace{-0.2cm}
        \caption{Maximum number of anomalies in missed $L^{(i)}$}
        \label{fig:maxval}
    \end{subfigure}
            \vspace{-0.2cm}
    \caption{X-axis denotes the fraction of samples $S/N$ and y-axis represents the fraction of anomalies $K/N$. }
    \label{fig:2d}\vspace{-0.6cm}
\end{figure}

\section{Conclusion}
We introduced a graph-based random sampling and consensus approach to effectively  detect  anomalous nodes in large-scale graphs. Rigorous analysis provides performance guarantees for our novel algorithm, by bounding the number of random draws involved. GraphSAC outperforms competing algorithms in detecting random walk-based anomalies, clustered anomalies, as well as contemporary adversarial attacks for graph data. Our future research will leverage GraphSAC to guard semi-supervised learning algorithms from adversarial attacks.
\small    
\bibliographystyle{IEEEtran}
\bibliography{my_bibliography}
\noindent
\appendix
\onecolumn
{\centering\huge \textbf{GraphSAC: Detecting anomalies in large-scale graphs}

~~~~~~~~~~~~~~~~~~~~~~~~~~~~~~~~~~~~~~~~~~\textbf{Supplementary material}}
\section{Proof of Theorem 1}
First, a link between the GraphSAC sampling scheme \eqref{eq:grsamplalt} and the GraphSAC filter \eqref{eq:graphsacdetect} is established.\footnote{Equations ($k$) with $k\le21$ correspond to the orginal manuscript while ($k$) with $k>21$ correspond to the supplementary material).}
Using the law of total probability (l.t.p), the total probability of sampling any contaminated subset is \begin{align}
\label{eq:sampltp}
\sum_{\mathcal{L}\in \bar{\mathcal{L}}_S^c} p_G(\mathcal{L})= |\bar{\mathcal{L}}_S^c|f.
\end{align}
This probability coincides with \eqref{eq:graphsacdetect}, the probability of declaring all contaminated samples as clean. Using the l.t.p. and the uniform sampling strategy of \eqref{eq:samplgrsac} it holds that
\begin{align}
\label{eq:detltp}
\sum_{\mathcal{L}\in \bar{\mathcal{L}}_S^c}p(\delta_f(\mathcal{L})=1|\mathcal{L}\in \bar{\mathcal{L}}^c_S)p(\mathcal{L}\in \bar{\mathcal{L}}^c_S)=|\bar{\mathcal{L}}_S^c|p(\delta_f(\mathcal{L})=1|\mathcal{L}\in \bar{\mathcal{L}}^c_S)p(\mathcal{L}\in \bar{\mathcal{L}}^c_S)
\end{align}
where $p(\mathcal{L}\in \bar{\mathcal{L}}^c_S)=|\bar{\mathcal{L}}^c_S|/|{\mathcal{L}}_S|$ due to the uniform sampling \eqref{eq:samplgrsac}. By introducing the probability of false alarm for   $\eqref{eq:graphsacdetect}$ as  $ p_{fa}:=p(\delta_f(\mathcal{L})=1|\mathcal{L}\in \bar{\mathcal{L}}^c_S)$ and equating the right side in \eqref{eq:sampltp} and  \eqref{eq:detltp} it holds that
\begin{align}
\label{eq:fa_def}
p_{fa}:=\frac{|{\mathcal{L}}_S|}{|\bar{{\mathcal{L}}}_S^c|}f  
\end{align}
Next, using the l.t.p. it holds for $p_g{(\mathcal{L})}$ \eqref{eq:pgrsampfin} that 
\begin{align}\label{eq:ltprob}
    1&=|\bar{\mathcal{L}}_S|d+ |\bar{\mathcal{L}}_S^c|f\\
    d&=\frac{1-|\bar{\mathcal{L}}_S^c|f}{|\bar{\mathcal{L}}_S|}
\end{align}
Furthermore by the definition of $\mathbf{P}_G$, it follows that 
\begin{align}
    \mathbf{P}_G=&\mathbb{E}_{\mathcal{L}\sim p_G}[ f(\mathcal{L})]\\
    =&\sum_{\mathcal{L}\in \mathcal{L}_S}p_g(\mathcal{L})f(\mathcal{L})\\\label{eq:expand}
    =&\sum_{\mathcal{L}\in \bar{\mathcal{L}}_S} f(\mathcal{L})d+
    \sum_{\mathcal{L}'\in \bar{\mathcal{L}}^c_S} f(\mathcal{L}')f\\\label{eq:expapl}
    =&d|\bar{\mathcal{L}}_S|\mathbf{P}_N+
    f|\bar{\mathcal{L}}^c_S|\mathbf{P}_A
\end{align}
where \eqref{eq:expand}  holds since  $\bar{\mathcal{L}}_S$ and $\bar{\mathcal{L}}^c_S$ are disjoint subsets,  \eqref{eq:expapl} follows by definition of $\mathbf{P}_N={1}/{|\bar{\mathcal{L}}_S|}\sum_{\mathcal{L}\in \bar{\mathcal{L}}_S} f(\mathcal{L})$ and 
$\mathbf{P}_A={1}/{|\bar{\mathcal{L}}^c_S|}\sum_{\mathcal{L}'\in \bar{\mathcal{L}}^c_S} f(\mathcal{L}')$. Using \eqref{eq:ltprob} and \eqref{eq:fa_def} it holds that
\begin{align}
\label{eq:partrespfa}
     \mathbf{P}_G=(1-\frac{|\bar{\mathcal{L}}_S^c|^2}{|{\mathcal{L}}_S|}p_{fa})
     \mathbf{P}_N+
 \frac{|\bar{\mathcal{L}}_S^c|^2}{|{\mathcal{L}}_S|}p_{fa}\mathbf{P}_A.
\end{align}
Hence, $\mathbf{P}_G$ is a convex combination between the nominal label distribution $\mathbf{P}_N$ and the one affected by the anomalies $\mathbf{P}_A$. By using \eqref{eq:partrespfa} the result of Theorem 1 follows
\begin{align}
  \|\mathbf{P}_G- \mathbf{P}_{N}\|_1
    =
\frac{|\bar{\mathcal{L}}_S^c|^2}{|{\mathcal{L}}_S|}p_{fa}\|\mathbf{P}_A -\mathbf{P}_{N}\|_1.
\end{align}
For the cardinality of the involved sets of subsets counting theory will be used. Specifically, it follows
\begin{align}
    |\bar{\mathcal{L}}_S|={N-K \choose S}
\end{align} since this is the number of ways to choose $S$ nodes out of $N-K$. Similarly, by counting the subsets with size $S$ constructed out of $N$ nodes it holds that
\begin{align}
    |{\mathcal{L}}_S|={N\choose S}.
\end{align} 
Finally, since the subsets $\bar{\mathcal{L}}^c_S$ and $\bar{\mathcal{L}}_S$ are not overlapping it holds that
$|\bar{\mathcal{L}}^c_S|=|{\mathcal{L}}_S|
-|\bar{\mathcal{L}}_S|$ and \begin{align}
|\bar{\mathcal{L}}^c_S|={N\choose S} -{N-K\choose S}.
\end{align}

\section{Proof of Theorem 2}
In this section $\|\mathbf{P}\|_1$ represents the norm 1 of matrix $\mathbf{P}$, and not the sum of absolute values of the elements of $\mathbf{P}$. Notice that since each draw of nodes $\mathcal{L}^{(i)}$ is independent, $\mathbf{P}_{G}^{(i)}$ is also independent across $i$. For each $\mathbf{P}_{G}^{(i)}$ it further holds
\begin{align}
\|\mathbf{P}_{G}^{(i)}\|\le
\sqrt{\|\mathbf{P}_{G}^{(i)}\|_1
\|\mathbf{P}_{G}^{(i)}\|_\infty}\le{\sqrt{N}}
\end{align}
where the last inequality follows since $[{\mathbf{P}_{G}^{(i)}}]_{n,c}\in[0,1], \forall n,c$ and $\sum_{c=1}^C[{\mathbf{P}_{G}^{(i)}}]_{n,c}=1$, $\forall n$ that implies $\|\mathbf{P}_{G}^{(i)}\|_1\le1$ 
and $\|\mathbf{P}_{G}^{(i)}\|_\infty\le N$. Next, the second moment of the  
matrix $\mathbf{P}_{G}^{(i)}$ is also bounded
\begin{align}
m_2(\mathbf{P}_{G}^{(i)})=max\big[\|\mathbb{E}\{
        \mathbf{P}_{G}^{(i)}{\mathbf{P}_{G}^{(i)}}^{\top}
        \}\|,\|\mathbb{E}\{{\mathbf{P}_{G}^{(i)}}^{\top}{\mathbf{P}_{G}^{(i)}}\}\|\big].
\end{align}
Let us define $\mathbf M= \mathbb{E}\{{\mathbf{P}_{G}^{(i)}}{\mathbf{P}_{G}^{(i)}}^{\top}\}$. For the spectral norm it holds that
\begin{align}
\|\mathbf{M}\|\le\sqrt{\|\mathbf{M}\|_\infty\|\mathbf{M}\|_1}\le N
\end{align}
where the last inequality holds since $M_{(n,n')}=\mathbb{E}\{\sum_{k=1}^{C}[{\mathbf{P}_{G}^{(i)}}]_{n,c}[{\mathbf{P}_{G}^{(i)}}]_{n',c}\}\le 1$ since $[{\mathbf{P}_{G}^{(i)}}]_{n,c}\in[0,1], \forall n,c$ and $\sum_{c=1}^C[{\mathbf{P}_{G}^{(i)}}]_{n,c}=1$, $\forall n$, and thus $\|\mathbf{M}\|_\infty\le N$ and $\|\mathbf{M}\|_1\le N$. Similarly, it can be shown that $\|\mathbb{E}\{\hat{\mathbf{P}}_{G}^{\top}\hat{\mathbf{P}}_{G}\}\|\le N$. Furthermore, the matrix variance is bounded as
\begin{align}
        v(\hat{\mathbf{P}}_{G})\le N
\end{align}
Hence, by appealing to the matrix Bernstein inequality for the sum of uncentered random matrices \cite[Ch. 6]{tropp2015introduction}
\eqref{eq:normbound} and \eqref{eq:propbound} in Theorem 2 follow.
\section{Proof of Corollary 1}
For the following consider that $\mathbf{Y}_{\mathcal{L}}$ is an $N\times C$ matrix whose $n$-th row is $\mathbf{y}_n^\transpose$ if $n\in\mathcal{L}$ and is $\mathbf{0}_{C}^\transpose$ otherwise.
First, for diffusion-based SSL models it holds that
\begin{align}
    \|\mathbf{P}_N
- \mathbf{P}_{A} \|_1 = \|h(\mathbf{A})(\frac{1}{|\bar{\mathcal{L}}_S|}\sum_{\mathcal{L}\in \bar{\mathcal{L}}_S}\mathbf{Y}_{\mathcal{L}}
-
\frac{1}{|\bar{\mathcal{L}}^c_S|}\sum_{\mathcal{L}'\in \bar{\mathcal{L}}^c_S}\mathbf{Y}_{\mathcal{L}'}) \|_1\label{eq:dif_th}
\end{align}
If $n\in \mathcal{N}$ is  nominal node, where $\mathcal{N}=\mathcal{V}-\mathcal{A}$, then $n$ is contained in ${N-K-1 \choose S-1}$ subsets of $\bar{\mathcal{L}}_S$ since $n$ may participate in a subset of size $S-1$ with any of the $N-K-1$ remaining nominal nodes. Hence,  each row $n$ of $\mathbf{Y}$ in the first sum of  \eqref{eq:dif_th}  is added ${N-K-1 \choose S-1}$ times if $n\in \mathcal{N}$  and 0 if $n\in\mathcal{A}$.   
Hence, it follows  
\begin{align}
    \frac{1}{|\bar{\mathcal{L}}_S|}\sum_{\mathcal{L}\in \bar{\mathcal{L}}_S}\mathbf{Y}_{\mathcal{L}}&=\frac{{N-K-1 \choose S-1}}{{N-K \choose S}}\mathbf{Y}_{\mathcal{N}}\\
    \label{eq:simpl_nom}
    &=\frac{S}{N-K}\mathbf{Y}_{\mathcal{N}}
\end{align}
For an anomalous node, $n\in \mathcal{A}$ is contained in $f_{\mathcal{A}}:={N-1 \choose S-1}$ subsets of $\bar{\mathcal{L}}_S^c$ since $n$ may participate in a subset of size $S-1$ with any of the $N-1$ remaining nodes. On the contrary, for a nominal node  $n\in \mathcal{N}$, the number of times $f_\mathcal{N}$, $n$ appears in subsets of $\bar{\mathcal{L}}_S^c$
can not be computed with straightforward counting theory, since any subset containing $n$ must contain at least one anomalous node. However, the total number of nodes in all subsets of $\bar{\mathcal{L}}^c_S$ is expressed as
\begin{align}\label{eq:total_nodes}
    S|\bar{\mathcal{L}}^c_S|=Kf_{\mathcal{A}} +(N-K)f_{\mathcal{N}}
\end{align}
where the left side  of \eqref{eq:total_nodes} follows since there are $|\bar{\mathcal{L}}^c_S|$ subsets each containing $S$ nodes and the right side of \eqref{eq:total_nodes} holds since there are $K (N-K)$ anomalous (nominal) nodes each included in $f_{\mathcal{A}} (f_{\mathcal{N}})$ subsets of $\bar{\mathcal{L}}^c_S$. Hence, 
it holds that
\begin{align}\label{eq:freq_nom_nodes}
f_{\mathcal{N}}=\frac{S}{N-K}|\bar{\mathcal{L}}^c_S|-
    \frac{K}{N-K}f_{\mathcal{A}}.
\end{align}
Hence, the $n$-th row of $\mathbf{Y}$ in  the second sum of  \eqref{eq:dif_th} is added $f_{\mathcal{A}}$ times if $n\in \mathcal{A}$ and $f_{\mathcal{N}}$ times if $n\in \mathcal{N}$. Thus, it follows 
\begin{align}
\label{eq:simpl_anom}
    \frac{1}{|\bar{\mathcal{L}}^c_S|}\sum_{\mathcal{L}'\in \bar{\mathcal{L}}^c_S}\mathbf{Y}_{\mathcal{L}'}=
    \frac{f_{\mathcal{N}}}{|\bar{\mathcal{L}}^c_S|}\mathbf{Y}_{\mathcal{N}}
    +
    \frac{f_{\mathcal{A}}}{|\bar{\mathcal{L}}^c_S|}
    \mathbf{Y}_{\mathcal{A}}
\end{align}
By utilizing \eqref{eq:simpl_nom} and \eqref{eq:simpl_anom} it holds that
\begin{align}
    \frac{1}{|\bar{\mathcal{L}}_S|}\sum_{\mathcal{L}\in \bar{\mathcal{L}}_S}\mathbf{Y}_{\mathcal{L}}
-
\frac{1}{|\bar{\mathcal{L}}^c_S|}\sum_{\mathcal{L}'\in \bar{\mathcal{L}}^c_S}\mathbf{Y}_{\mathcal{L}'}&=
\frac{S}{N-K}\mathbf{Y}_{\mathcal{N}}- \frac{f_{\mathcal{N}}}{|\bar{\mathcal{L}}^c_S|}\mathbf{Y}_{\mathcal{N}}-
    \frac{f_{\mathcal{A}}}{|\bar{\mathcal{L}}^c_S|}
    \mathbf{Y}_{\mathcal{A}}\\
    &=\frac{K f_{\mathcal{A}}}{|\bar{\mathcal{L}}^c_S|(N-K)}\mathbf{Y}_{\mathcal{N}}-\frac{f_{\mathcal{A}}}{|\bar{\mathcal{L}}^c_S|}
    \mathbf{Y}_{\mathcal{A}}\\
    &=\frac{f_{\mathcal{A}}}{|\bar{\mathcal{L}}^c_S|}
    \bigg(\frac{K}{N-K}\mathbf{Y}_{\mathcal{N}} -\mathbf{Y}_{\mathcal{A}}\bigg)
\end{align}
Hence, for diffusion-based SSL models it holds that 
\begin{align}
\label{eq:partres}
     \|\mathbf{P}_N
- \mathbf{P}_{A} \|_1=
\|
\frac{f_{\mathcal{A}}}{|\bar{\mathcal{L}}^c_S|}h(\mathbf{A})
    \bigg(\frac{K}{N-K}\mathbf{Y}_{\mathcal{N}} -\mathbf{Y}_{\mathcal{A}}\bigg) \|_1
\end{align}
Furthermore, using canonical vectors $\mathbf{Y}_{\mathcal{N}}$ is written as
\begin{align}
    \mathbf{Y}_{\mathcal{N}}=\sum_{n\in\mathcal{N}}
    \mathbf{e}_n\mathbf{y}_n^\transpose
\end{align}
where $\mathbf{e}_n$ is a canonical vector with 1 at the $n$-th position and otherwise 0. Hence, if $\mathbf{h}_n$ is the $n$-th column of $h(\mathbf{A})$ it follows from \eqref{eq:partres} that
\begin{align}
    \label{eq:partres1}
     \|\mathbf{P}_N
- \mathbf{P}_{A} \|_1=
\frac{f_{\mathcal{A}}}{|\bar{\mathcal{L}}^c_S|}
\|
\frac{K}{N-K}\sum_{n\in\mathcal{N}}
    \mathbf{h}_n\mathbf{y}_n^\transpose -\sum_{n\in\mathcal{A}}
    \mathbf{h}_n\mathbf{y}_n^\transpose \|_1.
\end{align}

\section{Alternative analytical guarantees provided by  GraphSAC}
This section provides alternative guarantees 
that directly account for $\delta_f(\cdot)$ in the analysis. 
The ultimate goal of GraphSAC is to obtain a class distribution per node that is not affected by the anomalous nodes but takes into account only the nominal nodes.
Hence, the desired probability matrix is 
\begin{align}
    \mathbf{P}_{N}&:=\mathbb{E}_{\bar{\mathcal{L}}}[f(\bar{\mathcal{L}})] \\
    \shortintertext{where $\bar{\mathcal{L}}$ are drawn uniformly $\bar{\mathcal{L}}\sim \mathrm{Unif}(\bar{\mathcal{L}}_S)$ from the set of ``clean'' $S-$size nodal subsets}
        \bar{\mathcal{L}}_S&:=\{ \mathcal{L}\subseteq \mathcal{V} : |\mathcal{L}| = S, \mathcal{L}\cap \mathcal{A} = \emptyset \}
\end{align}
Indeed, $ \mathbf{P}_{N}$ captures the nominal class distribution per node that is not affected by the anomalous nodes and conforms to the properties promoted by the SSL model.
Therefore, the largest entries in the anomaly score vector $\bm{\phi}(\mathbf{P}_{N})$ represent the anomalous nodes .

Unfortunately, direct estimation of  $\mathbf{P}_{N}$ is not feasible since  $\mathcal{A}$ is unknown. 
Further, note that if one directly accounts for all nodes will eventually include anomalous nodes that will bias the learned probability matrix.
On the other hand, GraphSAC builds incrementally an expected probability distribution using only samples that are minimally affected by anomalies.

GraphSAC decides if a sample is contaminated with anomalies by evaluating the accuracy of the instantiated SSL model~\eqref{eq:graphsacdetect}. 
If $\delta_f(\mathcal{L}^{(i)})=1$ then  $\hat{\mathbf{P}}^{(i)}_{G}$ is considered, and otherwise not. 
Ultimately, GraphSAC targets at the following expected probability matrix
\begin{align}
    \label{eq:pgraphsac2}
        \mathbf{P}_{G}:= \frac{\mathbb{E}_{\mathcal{L}}[ f(\mathcal{L})\delta_f(\mathcal{L})]}{\mathbb{E}_\mathcal{L}[\delta_f(\mathcal{L})]} 
\end{align}
where $\mathcal{L}$ is drawn uniformly from $\mathcal{L}_S$ \eqref{eq:samplgrsac}. If GraphSAC discards all the contaminated samples and maintains all the clean ones, then indeed ${\mathbb{E}_\mathcal{L}[\delta_f(\mathcal{L})]}=|\bar{\mathcal{L}}_S|/|{\mathcal{L}}_S|$ and $ \mathbb{E}_{\mathcal{L}}[ f(\mathcal{L})\delta_f(\mathcal{L})]= \sum_{\mathcal{L}\in
\bar{\mathcal{L}}_S}f(\mathcal{L})/|{\mathcal{L}}_S|$ that amounts to $
\mathbf{P}_G=\mathbf{P}_N$.

Hence, the distance among $\mathbf{P}_{G}$  and the desired $\mathbf{P}_{N}$ characterizes the performance of GraphSAC.  
Next, a couple of assumptions follow that enable theoretical claims stronger than Theorem 1.\footnote{Both assumption have been verified to hold empirically in simulation (See Fig.~\ref{fig:2d1})} \\
\noindent\textbf{1}) GraphSAC will identify clean samples 
    \begin{align}
    \label{eq:as1}
     \delta_f(\mathcal{L})=1~~~\text{if}~~~\mathcal{L}\in \bar{\mathcal{L}}_S
    \end{align}
\noindent\textbf{2}) GraphSAC will identify contaminated samples if they contain at least $K_m+1$ anomalies
    \begin{align}
    \label{eq:as2}
        \delta_f(\mathcal{L})=0~~~\text{if}~~~|\mathcal{L}\cap\mathcal{A}|\ge K_m+1
    \end{align}

The first assumption declares that sets with no anomalies will be labeled correctly, i.e. 
$p(\delta_f(\mathcal{L})=1|\mathcal{L}\in \bar{\mathcal{L}}_S)=1$. This can be satisfied by the required accuracy $T$ of Algorithm 1 low enough.

The second assumption asserts that GraphSAC will filter out a contaminated set if enough anomalies are in the sample. This is not surprising since a large number of anomalies in $\mathcal{L}$ will result to a small  consensus set $|\mathcal{U}^*|$ and $\mathcal{L}$ will be discarded. Thus, the probability of false alarm is
\begin{align}p_{fa}:=& p(\delta_f(\mathcal{L})=1|\mathcal{L}\notin \bar{\mathcal{L}}_S)\nonumber
\\
=& p(\delta_f(\mathcal{L})=1|\mathcal{L}\in {\mathcal{L}}_S^{K_m})
\end{align}
where ${\mathcal{L}}_S^{K_m}$ is the set of subsets containing at least one and at most $K_m$ anomalies, i.e. ${\mathcal{L}}_S^{K_m}:=\{ \mathcal{L}\subseteq \mathcal{V} : |\mathcal{L}| = S, 1\le|\mathcal{L}\cap \mathcal{A}| \le K_m\}$. 
\begin{theorem} Let $\mathbf{P}_A:=\mathbb{E}_{\mathcal{L}}\big[ f(\mathcal{L})\big|\delta_f(\mathcal{L})=1,\mathcal{L}\in
    \mathcal{L}_S^{K_m}\big]$ be the expected distribution obtained from $f(\cdot)$ given that at most $K_m$ anomalies are sampled and the filter declares the sample as clean,  $p_{\mathcal{L}_S^{K_m}}:=p(\mathcal{L} \in \mathcal{L}_S^{K_m})$,  and $p_{\bar{\mathcal{L}}_S}:=p(\mathcal{L}\in\bar{\mathcal{L}}_S)$.
The total variation distance between $\mathbf{P}_{G}$ and $\mathbf{P}_{N}$  is bounded as
\begin{align}
    \text{TV}(\mathbf{P}_{G},\mathbf{P}_{N})=\|\mathbf{P}_G- \mathbf{P}_{N}\|_1
    =
\frac{p_{fa}}{p_{\bar{\mathcal{L}}_S}+
p_{fa}
(1-p_{\bar{\mathcal{L}}_S})}\|\mathbf{P}_A 
    p_{\mathcal{L}_S^{K_m}}
    -\mathbf{P}_{N}(1-p_{\bar{\mathcal{L}}_S})\|_1.
\end{align}
\end{theorem}
\begin{proof}
First, by the law of total probability,  $\mathbb{E}_{\mathcal{L}}[ f(\mathcal{L})\delta_f(\mathcal{L})]$ is written as
\begin{align}\label{eq:decompltp}
    \mathbb{E}_{\mathcal{L}}\big[ f(\mathcal{L})\delta_f(\mathcal{L})\big]=
    \mathbb{E}_{\mathcal{L}}\big[ f(\mathcal{L})\delta_f(\mathcal{L})\big|\mathcal{L}\in
    \bar{\mathcal{L}}_S\big]p_{\bar{\mathcal{L}}_S}
    +
    \mathbb{E}_{\mathcal{L}}\big[ f(\mathcal{L})\delta_f(\mathcal{L})\big|\mathcal{L}\notin
    \bar{\mathcal{L}}_S\big](1-p_{\bar{\mathcal{L}}_S})
\end{align}
where $p_{\bar{\mathcal{L}}_S}:=p(\mathcal{L}\in\bar{\mathcal{L}}_S)$. Next, by applying \eqref{eq:as1} it holds that 
\begin{align}
\label{eq:cfirstas}
    \mathbb{E}_{\mathcal{L}}\big[ f(\mathcal{L})\delta_f(\mathcal{L})\big|\mathcal{L}\in
    \bar{\mathcal{L}}_S\big]=\mathbb{E}_{\bar{\mathcal{L}}}(f(\bar{\mathcal{L}}))=\mathbf{P}_{N}.
\end{align}
Furthermore, by using \eqref{eq:as2} it follows that 
\begin{align}
\label{eq:csecondas}
    \mathbb{E}_{\mathcal{L}}\big[ f(\mathcal{L})\delta_f(\mathcal{L})\big|\mathcal{L}\notin
    \bar{\mathcal{L}}_S\big]=\mathbb{E}_{\mathcal{L}}\big[ f(\mathcal{L})\delta_f(\mathcal{L})\big|\mathcal{L}\in
    \mathcal{L}_S^{K_m}\big] p(\mathcal{L}\in \mathcal{L}_S^{K_m}\big|\mathcal{L}\notin\bar{\mathcal{L}}_S)
\end{align}
The equality in \eqref{eq:csecondas} implies that GraphSAC always filters out samples corrupted with at least $K_m+1$ adversaries.  Hence, by using \eqref{eq:cfirstas} and \eqref{eq:csecondas}, it follows from \eqref{eq:decompltp} that 
\begin{align}
\label{eq:partaprox}
\mathbb{E}_{\mathcal{L}}\big[ f(\mathcal{L})\delta_f(\mathcal{L})\big]=
   \mathbf{P}_{N}p_{\bar{\mathcal{L}}_S}
    +
    \mathbb{E}_{\mathcal{L}}\big[ f(\mathcal{L})\delta_f(\mathcal{L})\big|\mathcal{L}\in
    \mathcal{L}_S^{K_m}\big] p_{\mathcal{L}_S^{K_m}}
\end{align}
where $p_{ \mathcal{L}_S^{K_m}}=p(\mathcal{L}\in \mathcal{L}_S^{K_m}\big|\mathcal{L}\notin\bar{\mathcal{L}}_S)(1-p_{\bar{\mathcal{L}}_S})=p(\mathcal{L}\in\mathcal{L}_S^{K_m})$. Furthermore, since $\mathbf{P}_G=\mathbb{E}_{\mathcal{L}}\big[ f(\mathcal{L})\delta_f(\mathcal{L})\big]/
\mathbb{E}_\mathcal{L}[\delta_f(\mathcal{L})]$ with
$p_\delta:=
\mathbb{E}_\mathcal{L}[\delta_f(\mathcal{L})]$ it holds that
\begin{align}
\|\mathbf{P}_G- \mathbf{P}_{N}\|_1&= \|\frac{1}{p_\delta} \big(\mathbf{P}_{N}p_{\bar{\mathcal{L}}_S}
    +
    \mathbb{E}_{\mathcal{L}}\big[ f(\mathcal{L})\delta_f(\mathcal{L})\big|\mathcal{L}\in
    \mathcal{L}_S^{K_m}\big] p_{\mathcal{L}_S^{K_m}}\big)- \mathbf{P}_{N}\|_1\\
    &=\|\mathbb{E}_{\mathcal{L}}\big[ f(\mathcal{L})\delta_f(\mathcal{L})\big|\mathcal{L}\in
    \mathcal{L}_S^{K_m}\big]
    \frac{p_{\mathcal{L}_S^{K_m}}}{p_\delta} 
    -\mathbf{P}_{N}\frac{p_{\delta}-p_{\bar{\mathcal{L}}_S}}{p_\delta} \|_1
\end{align}
Furthermore, using the law of total probability it holds that
\begin{align}
\mathbb{E}_{\mathcal{L}}\big[ f(\mathcal{L})\delta_f(\mathcal{L})\big|\mathcal{L}\in\mathcal{L}_S^{K_m}\big]&=
\mathbb{E}_{\mathcal{L}}\big[ f(\mathcal{L})\big|\delta_f(\mathcal{L})=1,\mathcal{L}\in
    \mathcal{L}_S^{K_m}\big]p(\delta_f(\mathcal{L})=1|\mathcal{L}\in \mathcal{L}^{K_m}_S)\\
    &=\mathbf{P}_A p_{fa}
\end{align}
where $p_{fa}=p(\delta_f(\mathcal{L})=1|\mathcal{L}\in \mathcal{L}^{K_m}_S)$ and $\mathbf{P}_A=\mathbb{E}_{\mathcal{L}}\big[ f(\mathcal{L})\big|\delta_f(\mathcal{L})=1,\mathcal{L}\in
    \mathcal{L}_S^{K_m}\big]$.
Following a similar argument it holds for $p_\delta$ that
\begin{align}
p_\delta=&
p(\delta_f(\mathcal{L})=1|\mathcal{L}\in \bar{\mathcal{L}}_S)p_{\bar{\mathcal{L}}_S}+
p(\delta_f(\mathcal{L})=1|\mathcal{L}\notin \bar{\mathcal{L}}_S)(1-p_{\bar{\mathcal{L}}_S})\label{eq:ltotprob}\\
=&p_{\bar{\mathcal{L}}_S}+
p(\delta_f(\mathcal{L})=1|\mathcal{L}\in \mathcal{L}^{K_m}_S)
(1-p_{\bar{\mathcal{L}}_S})\label{eq:assapplied}
\\=&p_{\bar{\mathcal{L}}_S}+
p_{fa}
(1-p_{\bar{\mathcal{L}}_S})
\end{align}
where \eqref{eq:ltotprob} follows from the law of total probability and \eqref{eq:assapplied} is a result of applying \eqref{eq:as1} and \eqref{eq:as2}. Hence, the result in Theorem 3 is established  as follows 
\begin{align}
\|\mathbf{P}_G- \mathbf{P}_{N}\|_1&= \|\mathbf{P}_A 
    \frac{p_{fa}p_{\mathcal{L}_S^{K_m}}}{p_{\bar{\mathcal{L}}_S}+
p_{fa}
(1-p_{\bar{\mathcal{L}}_S})} 
    -\mathbf{P}_{N}\frac{p_{fa}(1-p_{\bar{\mathcal{L}}_S}))}{p_{\bar{\mathcal{L}}_S}+
p_{fa}
(1-p_{\bar{\mathcal{L}}_S})} \|_1\\&=
\frac{p_{fa}}{p_{\bar{\mathcal{L}}_S}+
p_{fa}
(1-p_{\bar{\mathcal{L}}_S})}\|\mathbf{P}_A 
    p_{\mathcal{L}_S^{K_m}}
    -\mathbf{P}_{N}(1-p_{\bar{\mathcal{L}}_S})\|_1
\end{align}
\end{proof}
Theorem  3  states that the distance between the desired and the pursed distribution is  bounded by the scaled distance between the nominal distribution and the one affected by the anomalies. Different than Theorem 1 in the main document, Theorem 3, utilizes two additional (experimentally verified) assumptions, and provides a finer representation of $\|\mathbf{P}_G- \mathbf{P}_{N}\|_1$. 

\section{Runtime comparisons}
This section reports the runtime performance of the algorithms in Fig.~\ref{fig:labpert} in the main document. 
The scalabilty of GraphSAC is reflected on the runtime comparisons listed in Fig. \ref{fig:runtime}. Fig.~\ref{fig:runtime} reports the runtime per algorithm for $K=20$ in Fig.~\ref{fig:labpert}.  All experiments were run on a machine with i7-4790 @3.60 Ghz CPU, and 32GB of RAM. We used the Matlab and Python implementations provided by the authors of the compared algorithms. For fair comparison, the iterative version of Algorithm 1 is employed, even though the algorithm can be performed in parallel per $i$.

GraphSAC employs efficient SSL solvers,  and as expected is orders of magnitude faster than the competing anomaly detection approaches. 
The computational overhead associated with scoring each egonet~\cite{perozzi2016scalable}, subspace extraction~\cite{li2017radar} and deep neural networks~\cite{ding2019deep} limits the applicability of Amen~\cite{perozzi2016scalable}, Radar~\cite{li2017radar}, Gae~\cite{ding2019deep}, respectively to large-scale graphs.

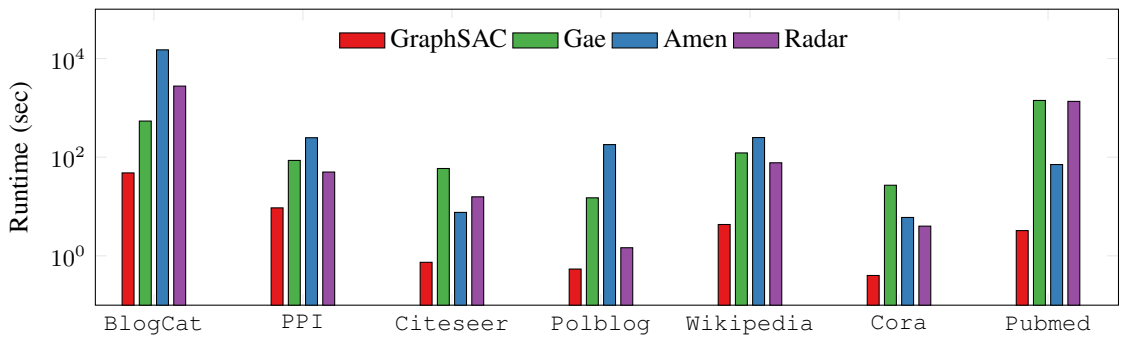
\begin{figure}[t]
%
%
%
\begin{tikzpicture}

\begin{axis}[%
width=2.5\mywidth,
height=0.987\myheight,
at={(0\mywidth,0\myheight)},
scale only axis,
bar shift auto,ticklabel style={font=\small},
log origin=infty,
xmin=0.509090909090909,
xmax=9.090909090909,
xtick={1,2.25,3.5,4.75,6,7.25,8.5},
xticklabels={{\texttt{BlogCat}},{\texttt{PPI}},{\texttt{Citeseer}},{\texttt{Polblog}}, {\texttt{Wikipedia}},{\texttt{Cora}},{\texttt{Pubmed}}},
ymode=log,
ymin=0.1,
ymax=100000,
yminorticks=true,
ylabel style={font=\color{white!15!black}},
ylabel={Runtime (sec) },
ylabel near ticks,
axis background/.style={fill=white},
 legend columns=4,
legend style={at={(0.23,0.82)}, anchor=south west, legend cell align=left, align=left, draw=none}
]

\addplot[ybar, bar width=0.1, fill=GraphSAC, draw=black, area legend] table[row sep=crcr] {%
	1	48\\
	2.25 9.39 \\
	3.5 0.74\\
	4.75 0.54 \\
	6 4.32\\
	7.25 	0.40\\
	8.5 	3.2497780323028564 \\
};
\addlegendentry{GraphSAC}

\addplot[ybar, bar width=0.1, fill=Gae, draw=black, area legend] table[row sep=crcr] {%
	1	540\\
	2.25 86\\
	3.5 59\\
	4.75 15\\
	6 122\\
	7.25 27\\
	8.5 1412.8371877670288\\
};
\addlegendentry{Gae}

\addplot[ybar, bar width=0.1, fill=Amen, draw=black, area legend] table[row sep=crcr] {%
1	14993\\
2.25 248\\
3.5 7.6\\
4.75 180\\
6 250\\
7.25 6\\
8.5 71\\
};
\addlegendentry{Amen}
\addplot[ybar, bar width=0.1, fill=Radar, draw=black, area legend] table[row sep=crcr] {%
	1	2761\\
	2.25 50\\
	3.5 15.72\\
	4.75 1.46\\
	6 77\\
	7.25 4\\
	8.5 1353\\
};
\addlegendentry{Radar}

\end{axis}
\end{tikzpicture}%
	\caption{Runtime comparisons across different graphs for Fig.~\ref{fig:labpert}.}\label{fig:runtime}
\end{figure}
\section{Additional Experiments}

\begin{figure}
		    \centering
		    \includegraphics
		    [width =0.45\linewidth, height=0.16\linewidth]
		    {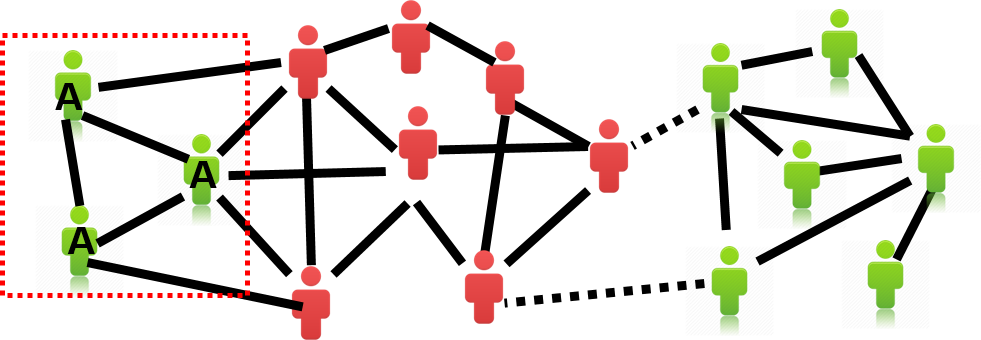}
		    \caption{Clustered anomalies.}
		    \label{fig:clustanom}
\end{figure}
This section provides additional experiments that characterize the performance of GraphSAC. Anomaly detection is tested under the \emph{structural} anomaly generation model based on random walks.  Furthermore, assumptions 1,2 of Theorem 3 are examined experimentally. Finally, the anomaly detection performance of GraphSAC is tested under varying the number of anomalies $K$ and the size of the sample set $S$. 

\textbf{Baselines}. See Section 4 in the main document. 

\textbf{Random walk-based structural anomalies}. First, a subset of $|\mathcal{A}|$ nodes is selected at random. For each $n \in \mathcal{A}$, a random walk of length $k$  is performed and $n$ is connected to the landing node. This process is repeated $I$ times per node $n$ that results to at most $I$ new connections to node $n$.

Table~\ref{tab:polblogstruct} reports the AUC values for detecting random walk-based structural anomalies with $I=5$.  The performance of the majority of the methods improves as  $k$ increases. This is expected since, for larger length of the random walks  $k$, the landing node will be increasingly dissimilar to the starting node and hence the structural anomaly will be easier identified.  For all values of $k$ GraphSAC outperforms the competing approaches.

\begin{table}[]
    \centering
    \caption{Discover anomalous nodes with $|\mathcal{A}|=20$ for the Polblogs dataset.} 
    \vspace{0.2cm}
    \rowcolors[]{1}{white}{gray}
    \begin{tabular}{p{3cm} c c c c}
    \toprule
Length of RW & $k=5$ & $k=10$ & $k=20$& $k=30$\\
    \midrule
    GraphSAC & \textbf{0.81} &\textbf{0.85}
&\textbf{0.83}&\textbf{0.91}
 \\
    
    Amen& 0.39 &0.42 &
0.55 & 0.60
\\
  
    Radar & 0.68&
0.71& 0.67&0.70\\
    
    Avg degree & 0.68&
0.67&0.53 &0.51
\\
    
Cut ratio & 0.49&
0.56&0.53&0.49
\\
Flake & 0.53&
0.49&0.55&0.49
\\
Conductance & 0.18&
0.18& 0.17 &0.16
\\
  
  	\bottomrule
    \end{tabular}
      \label{tab:polblogstruct}
\end{table}

\textbf{On the dependence of GraphSAC to $K$ and $S$}. Random walk-based label anomalies for the Citesser dataset are generated as in Section 4 of the main document. GraphSAC uses $I=2000$ iterations and the accuracy threshold is $T=0.6$. 

Fig. \ref{fig:auc2d1} reports the AUC performance of GraphSAC  for varying $K$ and $S$, where darker boxes indicate larger values. Observe that as $S$ increases the AUC increases as well. Although an increase on the number of anomalies $K$ decreases the performance in AUC, this can be rectified by an increase in the sampling size $S$. 

Fig. \ref{fig:pfa1} represents $p_c:=1-p_{fa}$ the percentage of contaminated subsets $\mathcal{L}^{(i)}$ that were discarded, which relates to assumption 1 of Theorem 3. Hence, large values of $p_c$ represent small $p_{fa}$ and are desirable. For small $S$, $p_c$ is large and as $S$ increases $p_c$ decreases. Further, note that a small $p_c$ (large $p_{fa}$) is not affecting the AUC performance.  Even though there are contaminated samples, the number anomalies in $\mathcal{L}$ is so small that does not affect the AUC values~Fig. \ref{fig:auc2d}.

Fig. \ref{fig:maxval1} shows the maximum number of anomalies $\max_i K^{(i)}$ that existed in the contaminated samples that were missclasiffied as nominal $\delta(f(\mathcal{L}^{(i)}))=1$ where $K^{(i)}:=|\mathcal{L}^{(i)}\cup\mathcal{A}|$. This number is equivalent to $K_m$ in assumption 2. Notice that assumption 2 indeed holds here since $K_m$ in Fig. \ref{fig:maxval1} is much less than the total number of anomalies that is $K=17,33,83,166,266$ respectively. 
\begin{figure}[htbp]
    \centering
    \begin{subfigure}[b]{0.49\textwidth}
        
        \includegraphics[width=0.5\textwidth]{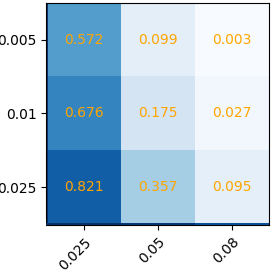}
        \caption{Percentage of contaminated subsets $\mathcal{L}^{(i)}$ discarded.}
        \label{fig:pfa1}
    \end{subfigure}
    \begin{subfigure}[b]{0.49\textwidth}
        \includegraphics[width=0.5\textwidth]{figs/fi_max_num_anomalexp10_score_val0.76_5_5anomal_rw_k=10_min_d=5_pct=0.08_citeseer.png}
        \caption{Maximum number of anomalies in $L^{(i)}$, where $\delta(f(L^{(i)}))=1$. or $K_m$}
        \label{fig:maxval1}
    \end{subfigure}
    \begin{subfigure}[b]{0.49\textwidth}
        \includegraphics[width=0.5\textwidth]{figs/fi_aucexp10_score_val0.76_5_5anomal_rw_k=10_min_d=5_pct=0.08_citeseer.png}
        \caption{AUC values for detecting the anomalous nodes.}
        \label{fig:auc2d1}
    \end{subfigure}
    \caption{The horizontal axis represents the fraction of samples $S/N$ and the vertical axis represents the fraction of anomalies $K/N$, where both are normalized by the number of nodes. }
    \label{fig:2d1}
\end{figure}
\end{document}